\documentclass[11pt]{article}
\usepackage[main,final]{neurips_2026}
\makeatletter
\@ifundefined{@notice}{}{\gdef\@notice{}}
\makeatother

\usepackage{graphicx}
\usepackage{url}
\usepackage{multirow}
\usepackage{booktabs}
\usepackage{colortbl}
\usepackage{xcolor}
\usepackage{amsmath, amssymb, amsthm}
\usepackage{array}
\usepackage[colorlinks=true,linkcolor=blue,citecolor=blue,urlcolor=blue]{hyperref}
\usepackage{microtype}
\usepackage{placeins}  

\newtheorem{lemma}{Lemma}
\theoremstyle{definition}

\newtheorem{problem}{Problem}
\theoremstyle{plain}
\newtheorem{theorem}{Theorem}
\newtheorem{proposition}{Proposition}

\theoremstyle{remark}
\newtheorem{remark}{Remark}
\newtheorem*{theorem*}{Theorem}
\newtheorem*{lemma*}{Lemma}

\newcommand{\best}[1]{\textbf{#1}}
\definecolor{bestcolor}{RGB}{220,237,200}
\definecolor{ourscolor}{RGB}{232,245,255}
\definecolor{headercolor}{RGB}{240,240,240}

\begin{document}

\title{A Separation Principle for Adversarial Attacks and Defense of Graph Neural Networks}

\author{
  Farid Mafi \\
  University of Waterloo \\
  Waterloo, ON, Canada \\
  \texttt{farid.mafishourestani@uwaterloo.ca}
  \And
  Saber Jafarpour \\
  University of Colorado Boulder \\
  Boulder, CO, USA \\
  \texttt{saber.jafarpour@colorado.edu}
  \AND
  Mohammad Pirani \\
  University of Ottawa \\
  Ottawa, ON, Canada \\
  \texttt{mpirani@uottawa.ca}
}

\maketitle

\begin{abstract}
Adversarial attacks on graph neural networks (GNNs) have been studied as a model-centric optimization problem driven by gradients, learned policies, or meta-learning, with the graph treated as a passive substrate. This view overlooks a defining feature of GNNs: their predictions are produced by structured message passing along the graph, so their adversarial vulnerabilities are encoded as much in the geometry of the graph as in the learned weights. We develop a structure-driven framework for feature attacks in which the attack is derived from the spectral properties of the normalized adjacency operator. Its central result is a \emph{separation principle}: under a natural relaxation of the attack objective, the optimal feature perturbation is rank one and factorizes into a graph-dependent component (degree-proportional spatial allocation, depending only on the graph) and a weight-dependent component (top singular vector of the composed weights, depending only on the weights). Node degree thus emerges as the fundamental graph quantity governing adversarial vulnerability to feature perturbations. The resulting attack, which we call \textbf{SepAttack}, requires no gradient computation, carries analytical guarantees, achieves competitive or superior performance against established baselines such as Nettack and IG-Attack with up to ${\sim}70\times$ speedup in the targeted setting, and transfers near-perfectly across architectures on continuous-feature graphs because its graph-dependent component is architecture-invariant by construction. We further derive a separation-principle adversarial training procedure (\textbf{SepAT}) directly from the closed-form worst-case attack and report robustness gains on binary-feature benchmarks.
\end{abstract}

\section{Introduction}
\label{sec:introduction}

Adversarial attacks on graph neural networks (GNNs) have been studied as a model-centric optimization problem: gradient procedures~\cite{zugner2018adversarial,xu2019topology,wu2019adversarial}, reinforcement-learning policies~\cite{dai2018adversarial,ju2023let}, meta-learning objectives~\cite{zugner2019adversarial}, and learned heuristics~\cite{zou2021tdgia,chen2022understanding} all attack the model, with the graph treated as fixed background. But a GNN's prediction is shaped by message passing along the graph as much as by the learned weights, so adversarial vulnerability cannot be a property of the model alone. How do the graph and the model combine at the adversarial optimum? Existing work analyzes one side at a time (spectral structure on the defense side, gradients on the attack side) or treats the interaction empirically. No closed-form principle has been articulated for how the two combine. This paper identifies one.

We call it the \textbf{separation principle}. Under the canonical feature-attack objective on an $L$-layer linearized GCN, the graph and the model contribute to the optimal perturbation \emph{independently} and combine \emph{multiplicatively}. The graph contributes a spatial allocation that depends only on its degree distribution, so higher-degree nodes receive proportionally larger perturbations regardless of model. The model contributes a single feature direction that depends only on its weights: the top singular vector of the composed weight matrix, regardless of graph. Neither contribution alone produces an effective attack; their product is the optimum (Theorem~\ref{thm:separation}). The principle has three constructive implications: the optimal attack admits a closed-form algorithm from a single eigendecomposition; attacks transfer across architectures to the extent that $\mathbf{w}_1$ is shared; and adversarial training reduces to enforcing output stability along a single analytically known direction.

We confirm this on standard citation networks. SepAttack dominates Nettack~\cite{zugner2018adversarial} and IG-Attack~\cite{wu2019adversarial} across all 20 targeted-evaluation cells on Cora and CiteSeer at up to ${\sim}70\times$ the speed of Nettack with zero gradient computation; on continuous-feature graphs it transfers at $97$--$104\%$ efficiency to GAT, GraphSAGE, APPNP, and JKNet without adaptation; and the principle-derived adversarial training procedure SepAT (Section~\ref{sec:robust_training}) recovers a substantial fraction of the accuracy lost to SepAttack on Cora and CiteSeer (Section~\ref{sec:sepat_results}). An ablation (Appendix~\ref{app:rank_one_validation}) tests the principle's distinctive claim: randomizing either factor reduces attack damage by $4$--$5\times$ on continuous-feature graphs and to the random baseline on binary-feature graphs, confirming both factors are jointly necessary.

\paragraph{Contributions.} (1)~The separation principle (Theorem~\ref{thm:separation}); (2)~the closed-form gradient-free attack SepAttack and its cross-architecture transfer (Sections~\ref{sec:feature_attack}--\ref{sec:transfer}); and (3)~principle-derived bilateral adversarial training SepAT (Section~\ref{sec:robust_training}).

\section{Related Work}
\label{sec:related_work}

Existing attacks on GNNs are predominantly model-centric. Gradient-based methods such as Nettack~\cite{zugner2018adversarial}, FGA~\cite{chen2018fast}, PGD~\cite{xu2019topology}, IG-Attack~\cite{wu2019adversarial}, and GraD~\cite{liu2023towards} score perturbations through model-specific gradients; meta-learning attacks (Metattack~\cite{zugner2019adversarial}) optimize perturbations as hyperparameters; reinforcement-learning approaches (RL-S2V~\cite{dai2018adversarial}, G2A2C~\cite{ju2023let}) and node-injection methods (NIPA~\cite{sun2020adversarial}, TDGIA~\cite{zou2021tdgia}, HAO~\cite{chen2022understanding}) learn attack policies empirically. Topological heuristics (GF-Attack~\cite{chang2020restricted}, VoteAttack~\cite{wen2024voteattack}, GOttack~\cite{alom2025gottack}) employ structural information only as auxiliary scoring signal. A spectral view of robustness is best developed on the \emph{defense} side: GCN-SVD~\cite{entezari2020all}, ProGNN~\cite{jin2020graph}, and the matrix-perturbation analysis of \citet{chang2021not} all exploit the spectral concentration of clean graphs. On the \emph{attack} side, the spectral line is sparser and uniformly iterative. \citet{bojchevski2019adversarial} target DeepWalk embeddings via eigenvalue perturbation; SPAC~\cite{lin2022graph} maximizes a Laplacian-spectrum distance via projected gradient descent; SAT~\cite{li2022spectral} performs adversarial training through gradient-driven eigenvector perturbations. The closest prior work (SPAC and SAT) constructs perturbations through iterative optimization and does not identify a closed-form optimal attack or a spatial--feature factorization. Our framework departs from this lineage in three respects: it targets feature rather than topology perturbations, it derives the optimal perturbation in closed form rather than iteratively, and it identifies an explicit rank-one factorization into graph-dependent and weight-dependent components that no prior spectral-attack work exposes. On adversarial training, \citet{mujkanovic2022defenses} and \citet{gosch2023adversarial} highlight the fragility of many GNN defenses under adaptive attacks; SepAT derives its training objective directly from the closed-form worst-case attack of Theorem~\ref{thm:separation}. Appendix~\ref{app:lit_comparison} provides an extended categorization with a comparison table.

\newpage

\section{Problem Formulation}
\label{sec:problem_formulation}

\subsection{Preliminaries}

Let $\mathcal{G} = (\mathcal{V}, \mathcal{E}, \mathbf{A}, \mathbf{X})$ be a graph with node features (an attributed graph), where $\mathcal{V} = \{v_1, \ldots, v_N\}$ is the set of nodes, $\mathcal{E} \subseteq \mathcal{V} \times \mathcal{V}$ is the set of edges, and $w : \mathcal{E} \to \mathbb{R}_{\ge 0}$ is a weight function. The (possibly weighted) adjacency matrix $\mathbf{A} \in \mathbb{R}^{N \times N}$ is defined as
\begin{equation*}
\mathbf{A}_{ij} =
\begin{cases}
w(v_i, v_j), & \text{if } (v_i, v_j) \in \mathcal{E}, \\
0, & \text{otherwise}.
\end{cases}
\end{equation*}
The framework allows arbitrary non-negative edge weights; the citation networks used throughout our experiments (Cora, CiteSeer, PubMed) have binary entries, but no step in the derivation exploits binarity of $\mathbf{A}$.
We define the node feature matrix $\mathbf{X} \in \mathbb{R}^{N \times D}$ as a stack of node-feature row vectors:
\begin{equation*}
\mathbf{X} \;=\; \begin{bmatrix}\,\mathbf{X}_{1,:}\, \\[-1pt] \,\vdots\, \\[-1pt] \,\mathbf{X}_{N,:}\,\end{bmatrix},
\qquad \mathbf{X}_{i,:} \in \mathbb{R}^D,
\end{equation*}
where the row $\mathbf{X}_{i,:}$ is the $D$-dimensional feature vector associated with node $v_i$.
Let $\tilde{\mathbf{A}} = \mathbf{A} + \mathbf{I}_N$ denote the adjacency with added self-loops and $\tilde{\mathbf{D}}$ the corresponding degree matrix; the (symmetric) normalized adjacency, denoted as $\hat{\mathbf{A}}$, is defined as $\hat{\mathbf{A}} = \tilde{\mathbf{D}}^{-1/2}\,\tilde{\mathbf{A}}\,\tilde{\mathbf{D}}^{-1/2}$.

The Graph Convolutional Network (GCN), introduced in~\cite{kipf2016semi}, is one of the most popular and widely used GNN architectures. The layer-wise propagation rule of an $L$-layer GCN is, for $\ell \in \{0, 1, \ldots, L-1\}$,
\begin{equation}
\label{eq:gcn_layerwise}
    \mathbf{H}^{(\ell+1)} \;=\; \sigma\!\left(\hat{\mathbf{A}}\,\mathbf{H}^{(\ell)}\,\mathbf{W}^{(\ell)}\right),
    \qquad \mathbf{H}^{(0)} = \mathbf{X},
\end{equation}
where $\mathbf{H}^{(\ell)} \in \mathbb{R}^{N \times F_\ell}$ collects the node representations at layer $\ell$, $\mathbf{W}^{(\ell)} \in \mathbb{R}^{F_\ell \times F_{\ell+1}}$ is the trainable weight matrix at layer $\ell$, and $\sigma(\cdot)$ is a component-wise non-linear activation. We adopt the standard GCN convention of omitting an additive bias term in~\eqref{eq:gcn_layerwise}; including a bias does not alter any of the spectral derivations that follow, since the bias cancels in the linearized residual of Section~\ref{sec:feature_attack}. The final node-classification output of an $L$-layer GCN is
\begin{equation}
\label{eq:gcn_Llayer}
    \mathbf{Z} \;=\; f_\theta(\mathbf{A}, \mathbf{X}) \;=\; \mathrm{softmax}\!\left(\hat{\mathbf{A}}\,\sigma\!\left(\hat{\mathbf{A}}\,\sigma\!\left(\cdots \hat{\mathbf{A}}\,\mathbf{X}\,\mathbf{W}^{(0)}\right)\mathbf{W}^{(1)}\right)\cdots\mathbf{W}^{(L-1)}\right).
\end{equation}

For the sake of generality we present all of the analysis below for an arbitrary number of layers $L$. In practice, two-layer GCNs ($L = 2$) are the most commonly used variant, as deeper GCNs tend to suffer from over-smoothing~\cite{li2018deeper, oono2019graph} (see Appendix~\ref{app:two_layer} for a detailed justification); this is also the architecture we use for the experiments of Section~\ref{sec:results}.

\subsection{Threat Model and Residual Analysis}
\label{sec:threat_model}

In this work, we focus on \emph{feature} perturbations and develop a closed-form spectral characterization of the worst-case feature attack in Section~\ref{sec:feature_attack}.

We model a feature attack as an additive perturbation $\mathbf{X} \mapsto \mathbf{X} + \Delta\mathbf{X}$ with $\Delta\mathbf{X} \in \mathbb{R}^{N \times D}$. Its effect on the model is captured by the \emph{feature residual} at the output layer, defined as the difference between the $L$-th layer representation of the perturbed input and that of the clean input:
\begin{equation}
\label{eq:residual_def}
    \mathbf{R}_f^{(L)}(\Delta\mathbf{X})
    \;\triangleq\; \mathbf{H}^{(L)}(\mathbf{X} + \Delta\mathbf{X}) \;-\; \mathbf{H}^{(L)}(\mathbf{X}).
\end{equation}
Intuitively, $\mathbf{R}_f^{(L)}(\Delta\mathbf{X})$ measures how much the attack shifts the pre-classification embedding,\footnote{The residual is taken \emph{before} the softmax of~\eqref{eq:gcn_Llayer}, i.e., on the pre-classification logits $\mathbf{H}^{(L)}$ rather than on the post-softmax probabilities $\mathbf{Z} = \mathrm{softmax}(\mathbf{H}^{(L)})$. Since the softmax is monotone in the logit gap, working with logit residuals is equivalent for our purposes (it preserves the argmax decision and the ranking of perturbations) and matches the canonical choice of prior work~\cite{zugner2018adversarial,wu2019adversarial}.} and therefore directly drives the change in the downstream classification loss. The worst-case feature attack on a binary-feature graph maximizes this residual under a feature-flip budget, which gives:

\begin{problem}[Worst-case feature attack on a nonlinear GCN]
\label{prob:nonlinear}
Given an $L$-layer GCN with nonlinear forward map $\mathbf{H}^{(L)}(\cdot)$ defined by~\eqref{eq:gcn_Llayer}, a clean binary feature matrix $\mathbf{X}\in\{0,1\}^{N\times D}$, and a total flip budget $\Gamma_{\mathrm{total}}\in\mathbb{N}$, find
\begin{equation}
\label{eq:problem_nonlinear}
\begin{aligned}
    \max_{\Delta\mathbf{X}}\quad & \bigl\|\mathbf{H}^{(L)}(\mathbf{X}+\Delta\mathbf{X}) \;-\; \mathbf{H}^{(L)}(\mathbf{X})\bigr\|_2 \\
    \mathrm{s.t.}\quad
    & \|\Delta\mathbf{X}\|_0 \;\le\; \Gamma_{\mathrm{total}}, \\
    & \mathbf{X} + \Delta\mathbf{X} \;\in\; \{0,1\}^{N \times D}.
\end{aligned}
\end{equation}
\end{problem}
The combined constraint $\mathbf{X}+\Delta\mathbf{X}\in\{0,1\}^{N\times D}$ together with $\mathbf{X}\in\{0,1\}^{N\times D}$ implies $\Delta\mathbf{X}\in\{-1,0,1\}^{N\times D}$: at each entry, only feasible flips ($0\!\to\!1$ or $1\!\to\!0$) are allowed. While we defined Problem~\ref{prob:nonlinear} for binary-feature graphs, for continuous-feature graphs (e.g., Cora-ML and PubMed), this formulation can be adapted by replacing the $\ell_0$-norm with the Frobenius norm to measure the budget. Problem~\ref{prob:nonlinear} is the worst-case feature attack \emph{as the adversary actually faces it}: the residual depends nonlinearly on $\Delta\mathbf{X}$ through the ReLU activations of $\mathbf{H}^{(L)}$, and the feasible set is combinatorial. Solving it exactly is therefore intractable; existing attacks~\cite{zugner2018adversarial,wu2019adversarial,xu2019topology} accordingly approach it via gradient-based heuristics or greedy procedures. We take a different route in Section~\ref{sec:feature_attack}, where we define a tractable surrogate by linearizing the ReLU and (on binary-feature graphs) replacing the discrete feature-flip constraint by an equivalent Frobenius budget, then solve the resulting continuous problem in closed form.

\section{Spectral Approach to Feature Attacks}
\label{sec:feature_attack}

\paragraph{From Problem~\ref{prob:nonlinear} to a tractable surrogate.}
We define a tractable surrogate by linearizing the ReLU activations of~\eqref{eq:gcn_layerwise}, following the standard surrogate construction of~\cite{zugner2018adversarial,wu2019adversarial}; the feature residual~\eqref{eq:residual_def} then admits the closed form
\begin{equation}
\label{eq:feature_residual}
    \mathbf{R}_f^{(L)}(\Delta\mathbf{X}) \;=\; \hat{\mathbf{A}}^L\,\Delta\mathbf{X}\,\mathbf{W},
\end{equation}
where $\mathbf{W} = \mathbf{W}^{(0)}\mathbf{W}^{(1)}\cdots\mathbf{W}^{(L-1)} \in \mathbb{R}^{D \times C}$ is the composed weight matrix and $\hat{\mathbf{A}}^L$ encodes the $L$-hop neighborhood structure through which perturbations propagate (full derivation in Appendix~\ref{app:residual}). For continuous-feature graphs (Cora-ML, PubMed), a Frobenius-norm budget on $\Delta\mathbf{X}$ is already the natural formulation. For binary-feature graphs (Cora, CiteSeer), every feasible $\Delta\mathbf{X}$ lies in $\{-1,0,1\}^{N\times D}$, so $\|\Delta\mathbf{X}\|_F^2 = \|\Delta\mathbf{X}\|_0 \le \Gamma_{\mathrm{total}}$ and replacing the discrete feature-flip constraint by a Frobenius bound is a genuine relaxation (Appendix~\ref{app:optimization}). Setting $\gamma\triangleq\Gamma_{\mathrm{total}}$ and using the equivalence $\max\|\cdot\|_2 \Leftrightarrow \max\|\cdot\|_2^2$ on a compact set, we obtain the \emph{linear cost, continuous feature attack problem}:\footnote{The two qualifiers refer to the simplifications relative to Problem~\ref{prob:nonlinear}: \emph{linear cost} (the residual is now linear in $\Delta\mathbf{X}$, since the ReLU non-linearities have been replaced by the identity), and \emph{continuous feature} (the discrete feature-flip set $\{-1,0,1\}^{N\times D}$ has been replaced by the continuous Frobenius ball).}
\begin{equation}
\label{eq:relaxed_problem}
\begin{aligned}
    \max_{\Delta \mathbf{X}} \quad & \|\hat{\mathbf{A}}^L \Delta \mathbf{X}\, \mathbf{W}\|_2^2 \\
    \mathrm{s.t.} \quad & \|\Delta \mathbf{X}\|_F^2 \leq \gamma.
\end{aligned}
\end{equation}
By feasible-set inclusion and the 1-Lipschitz property of ReLU, the optimal value of~\eqref{eq:relaxed_problem} upper-bounds that of Problem~\ref{prob:nonlinear}. The remainder of this section solves~\eqref{eq:relaxed_problem} in closed form via spectral decomposition; Appendix~\ref{app:rank_one_validation} validates empirically that the resulting closed-form attack remains the strongest rank-one attack we tested once the full ReLU non-linearity is reinstated.

\paragraph{Spectral decomposition.}
We expand the perturbation $\Delta\mathbf{X}$ in the eigenbasis of $\hat{\mathbf{A}}$. Because $\hat{\mathbf{A}}$ is symmetric, the eigenvectors of any (positive) power of $\hat{\mathbf{A}}$ coincide with those of $\hat{\mathbf{A}}$ itself, so working with $\hat{\mathbf{A}}^{2L}$ (which appears naturally after squaring the residual~\eqref{eq:feature_residual}) does not change the eigenbasis. Let $(\mu_1, \mathbf{v}_1), \ldots, (\mu_N, \mathbf{v}_N)$ and $(\lambda_1, \mathbf{w}_1), \ldots, (\lambda_D, \mathbf{w}_D)$ be the eigenvalue--eigenvector pairs of $\hat{\mathbf{A}}^{2L}$ and of the weight Gram matrix $\mathbf{W}\mathbf{W}^T$, respectively, ordered so that $\mu_1 \geq \cdots \geq \mu_N \geq 0$ and $\lambda_1 \geq \cdots \geq \lambda_D \geq 0$. Since $\{\mathbf{v}_1,\ldots,\mathbf{v}_N\}$ is an orthonormal basis of $\mathbb{R}^N$, every perturbation $\Delta\mathbf{X}\in\mathbb{R}^{N\times D}$ can be written uniquely as
\begin{equation}
\label{eq:eigen_expansion}
    \Delta\mathbf{X} \;=\; \sum_{i=1}^N \mathbf{v}_i\,\boldsymbol{\epsilon}_i^T,
    \qquad \boldsymbol{\epsilon}_i \in \mathbb{R}^D.
\end{equation}
Reading~\eqref{eq:eigen_expansion} column-wise, the $j$-th feature column of $\Delta\mathbf{X}$ is $\Delta\mathbf{X}_{:,j} = \sum_{i=1}^N \epsilon_i^{(j)}\,\mathbf{v}_i$, where $\boldsymbol{\epsilon}_i = (\epsilon_i^{(1)},\ldots,\epsilon_i^{(D)})^T \in \mathbb{R}^D$ collects, in its $j$-th entry, the coefficient of feature $j$ along eigenmode $i$. Substituting~\eqref{eq:eigen_expansion} into~\eqref{eq:relaxed_problem} reformulates the relaxed problem as an optimization over the new variables $\{\boldsymbol{\epsilon}_i\}_{i=1}^N$ (full derivation in Appendix~\ref{app:thm_proof}):
\begin{equation}
\label{eq:reformulated_final}
\begin{aligned}
    &\max_{\{\boldsymbol{\epsilon}_i\}_{i=1}^N} \quad \lambda_{\max}\!\left(\sum_{i=1}^{N} \mu_i \mathbf{W}^T \boldsymbol{\epsilon}_i \boldsymbol{\epsilon}_i^T \mathbf{W}\right) \\[6pt]
    &\text{s.t.} \quad \sum_{i=1}^{N} \|\boldsymbol{\epsilon}_i\|_2^2 \leq \gamma.
\end{aligned}
\end{equation}
The coefficients $\{\boldsymbol{\epsilon}_i\}_{i=1}^N$ are the new optimization variables. The graph and the weights now enter the objective in decoupled positions: the graph contributes only the scalar weights $\mu_i$ in front of each rank-one summand, while the weight matrix $\mathbf{W}$ multiplies all summands uniformly from outside. This reformulation is the starting point for the main result in Section~\ref{sec:main_result}.

\subsection{Separation Principle for Optimal Perturbation Structure}
\label{sec:main_result}

We can now state the central theoretical result of the paper, a closed-form characterization of the optimal feature perturbation; the full proof is given in Appendix~\ref{app:thm_proof}.

\begin{theorem}[Separation Principle for Optimal Perturbation Structure]
\label{thm:separation}
Consider an $L$-layer GCN defined by~\eqref{eq:gcn_layerwise} with the associated linear cost, continuous feature attack problem~\eqref{eq:relaxed_problem}. Suppose that the composed weight matrix is
$\mathbf{W} = \mathbf{W}^{(0)}\mathbf{W}^{(1)}\cdots\mathbf{W}^{(L-1)}$
and that $(\lambda_1, \mathbf{w}_1)$ is the dominant eigenpair of $\mathbf{W}\mathbf{W}^T$.
Then an optimal solution of~\eqref{eq:relaxed_problem} is given by
\begin{equation}
\label{eq:optimal_DX}
\boxed{\;
\Delta\mathbf{X}^\star \;=\; \sqrt{\gamma}\;
\frac{\tilde{\mathbf{D}}^{1/2}\,\mathbf{1}_N}{\|\tilde{\mathbf{D}}^{1/2}\,\mathbf{1}_N\|_2}\;\mathbf{w}_1^T,
\;}
\end{equation}
achieving the optimal value $\lambda_1\gamma$. The optimal perturbation is a rank-one matrix and admits a separation into a graph-dependent factor $\tilde{\mathbf{D}}^{1/2}\mathbf{1}_N/\|\tilde{\mathbf{D}}^{1/2}\mathbf{1}_N\|_2$, depending only on node degrees with per-node magnitude $\|\Delta\mathbf{X}^\star_{i,:}\|_2 \propto \sqrt{\tilde{d}_i}$, and a weight-dependent factor $\mathbf{w}_1$, depending only on the composed weights. The graph-dependent factor is independent of depth $L$; only the weight-dependent factor adapts with depth, through the change in $\mathbf{W}$.
\end{theorem}

\begin{proof}[Proof sketch]
The eigenbasis expansion $\Delta\mathbf{X}=\sum_i\mathbf{v}_i\boldsymbol{\epsilon}_i^\top$ in the basis of $\hat{\mathbf{A}}^{2L}$ reduces~\eqref{eq:reformulated_final} to maximizing $\lambda_{\max}(\sum_i\mu_i\mathbf{W}^T\boldsymbol{\epsilon}_i\boldsymbol{\epsilon}_i^T\mathbf{W})$ subject to $\sum_i\|\boldsymbol{\epsilon}_i\|_2^2\leq\gamma$. A Weyl + Rayleigh sandwich (Appendix~\ref{app:thm_proof}) shows the optimum is attained at $\boldsymbol{\epsilon}_1^\star=\sqrt{\gamma}\,\mathbf{w}_1$, $\boldsymbol{\epsilon}_i^\star=\mathbf{0}$ for $i\geq 2$, with optimal value $\lambda_1\mu_1\gamma$. The dominant eigenpair of $\hat{\mathbf{A}}^{2L}$ is identified in closed form via the spectral theory of the normalized Laplacian, giving $\mu_1=1$ (so $\lambda_1\mu_1\gamma = \lambda_1\gamma$) and $\mathbf{v}_1 = \tilde{\mathbf{D}}^{1/2}\mathbf{1}_N/\|\tilde{\mathbf{D}}^{1/2}\mathbf{1}_N\|_2$, independent of $L$ since powers of a symmetric matrix share its eigenvectors. Substituting these expressions back yields~\eqref{eq:optimal_DX}.
\end{proof}

\paragraph{Discussion.}
The factorization $\Delta\mathbf{X}^\star = \sqrt\gamma\,\mathbf{v}_1\mathbf{w}_1^\top$ has three constructive consequences. First, the optimum admits a closed-form expression from a single eigendecomposition, so problem~\eqref{eq:relaxed_problem} need not be solved by a gradient-based procedure as in PGD~\cite{xu2019topology}, IG-Attack~\cite{wu2019adversarial}, or Nettack~\cite{zugner2018adversarial}. Second, the optimal value $\lambda_1\gamma$ depends on the weight spectrum only through $\lambda_1$, since $\mu_1=1$ is fixed by the normalization of $\hat{\mathbf{A}}$. Third, the graph factor depends only on degrees and is therefore architecture-invariant, so the perturbation transfers across architectures to the extent that the top weight singular vector $\mathbf{w}_1$ is shared. The closed-form construction is also markedly cheaper than gradient-based baselines (Section~\ref{sec:targeted_eff}). The principle also offers a theoretical perspective on the empirical degree bias reported by GraD~\cite{liu2023towards}: degree-proportional allocation arises naturally from the optimum, not from gradient-estimation artifacts. Full discussion in Appendix~\ref{app:thm_discussion}.

When the attacker controls only a subset $\mathcal{A}\subsetneq\mathcal{V}$, the full-access optimum is infeasible. We restrict the degree-proportional allocation to $\mathcal{A}$:
\begin{proposition}[Restricted-Access Heuristic]
\label{prop:restricted}
Let $\mathcal{A}\subseteq\mathcal{V}$ and let $[\mathbf{b}_\mathcal{A}]_i=\sqrt{\tilde{d}_i}$ for $i\in\mathcal{A}$ and $0$ otherwise. Then $\Delta\mathbf{X}^{\mathcal{A}}=\sqrt{\gamma}\,(\mathbf{b}_\mathcal{A}/\|\mathbf{b}_\mathcal{A}\|_2)\,\mathbf{w}_1^\top$ is feasible, supported on $\mathcal{A}$, and preserves the graph--weight separation.
\end{proposition}
The continuous direction $\Delta\mathbf{X}^{\mathcal{A}}$ also yields a feasible integer-flip attack for Problem~\ref{prob:nonlinear} via greedy top-$\Gamma_{\text{total}}$ discretization on the outer-product score (full construction in Appendix~\ref{app:optimization}). Although we do not claim optimality of the attack proposed in Proposition~\ref{prop:restricted}, Appendix~\ref{app:rank_one_validation} (Figure~\ref{fig:spectral-validation}b) shows it consistently dominates uniform and random-support baselines.

We refer to the family of rank-one attacks derived from the separation principle --- the full-access optimum~\eqref{eq:optimal_DX} of Theorem~\ref{thm:separation} and the restricted-access version of Proposition~\ref{prop:restricted} --- as \textbf{SepAttack}. In both variants the spatial factor depends only on the graph (through node degrees), so the attack does not require gradient computation through the model.

\section{Adversarial Training}
\label{sec:robust_training}

Adversarial training~\cite{madry2018pgd} modifies the training procedure so that the trained network is robust to adversarial perturbations of the input, by replacing the standard empirical risk by its worst-case counterpart over an adversarial set. For feature attacks on a GCN of parameters $\theta$ with the Frobenius budget of Section~\ref{sec:threat_model}, this gives the standard min--max formulation
\begin{equation}
\label{eq:minmax}
    \min_\theta\;\mathbb{E}_{(\mathbf{A},\mathbf{X},\mathbf{y})}\Bigl[\max_{\|\Delta\mathbf{X}\|_F\le\delta}\,\mathcal{L}_{\mathrm{CE}}\bigl(f_\theta(\mathbf{A},\mathbf{X}+\Delta\mathbf{X}),\mathbf{y}\bigr)\Bigr].
\end{equation}
Solving the inner maximization exactly at every training step is intractable: it is itself an adversarial-attack problem against the network being trained, and the standard remedy --- approximating it with iterative gradient ascent (e.g., PGD-$k$) for $k$ steps inside every parameter update --- multiplies the per-step training cost by $k$ and is the dominant cost of adversarial training.

We propose a different route. Theorem~\ref{thm:separation} gives a closed-form expression for the worst-case attack $\Delta\mathbf{X}^\ast$ on the linearized model under the Frobenius budget, and we use this expression to approximate the inner maximization in~\eqref{eq:minmax} with no inner optimization at all. The closed-form worst case is the rank-one direction $\Delta\mathbf{X}^\ast=\delta\,\tilde{\mathbf{D}}^{1/2}\mathbf{1}_N/\|\tilde{\mathbf{D}}^{1/2}\mathbf{1}_N\|_2\cdot\mathbf{w}_1^\top$, where $\mathbf{w}_1$ is the top singular vector of $\mathbf{W}=\mathbf{W}^{(0)}\cdots\mathbf{W}^{(L-1)}$ at the current $\theta$. Since ReLU makes the model's response to $\pm\Delta\mathbf{X}^\ast$ asymmetric, we regularize \emph{bilaterally} along this direction:
\begin{equation}
\label{eq:robust_loss}
\begin{aligned}
    \mathcal{L}_{\mathrm{robust}} \;=\;& \mathcal{L}_{\mathrm{CE}}\bigl(f_\theta(\mathbf{A},\mathbf{X}),\mathbf{y}\bigr) + \lambda\|\mathbf{W}^{(0)}\|_F^2 \\
    & + \frac{\alpha}{|\mathcal{V}_L|}\!\sum_{u\in\mathcal{V}_L}\sum_{s\in\{+,-\}}\!\bigl\|f_\theta(\mathbf{A},\mathbf{X}+s\Delta\mathbf{X}^\ast)_u - f_\theta(\mathbf{A},\mathbf{X})_u\bigr\|_2^2.
\end{aligned}
\end{equation}
A natural alternative would be to penalize the spectral norm $\sigma_{\max}(\mathbf{W})$ directly, but this fails: shrinking $\mathbf{W}$ scales the adversarial residual and the clean classification margins by the same factor, so the attack-versus-margin ratio is unchanged (Appendix~\ref{app:naive_penalty}). Robustness must therefore be enforced at the model's output level, which is what~\eqref{eq:robust_loss} does. We refer to the resulting training procedure as \textbf{SepAT} (separation-principle adversarial training). SepAT recomputes $\Delta\mathbf{X}^\ast$ each step via SVD of the composed weights $\mathbf{W}=\mathbf{W}^{(0)}\cdots\mathbf{W}^{(L-1)}\in\mathbb{R}^{D\times C}$; one power-iteration step costs $O(DC)$ and is negligible relative to a forward/backward pass, even on our largest dataset (PubMed). Hyperparameters, the curriculum schedule, and the bilateral-regularization derivation are in Appendix~\ref{app:defense_derivation}; empirical robustness results are reported in Section~\ref{sec:sepat_results}.


\section{Numerical Experiments}
\label{sec:results}

We evaluate on four citation networks: Cora ($N{=}2{,}708$, $D{=}1{,}433$, $C{=}7$, binary BoW), Cora-ML ($N{=}2{,}810$, $D{=}2{,}879$, $C{=}7$, continuous), CiteSeer ($N{=}3{,}327$, $D{=}3{,}703$, $C{=}6$, binary BoW), and PubMed ($N{=}19{,}717$, $D{=}500$, $C{=}3$, TF-IDF), with three heterophilous WebKB benchmarks (Texas, Wisconsin, Cornell) used in Section~\ref{sec:heterophilous}. The target model is a 2-layer GCN~\cite{kipf2016semi} (16 hidden units, ReLU, dropout~0.5) trained on a 10/10/80 split following the Nettack protocol~\cite{zugner2018adversarial}; results are averaged over 5 seeds (mean $\pm$ std). Targeted baselines are Nettack~\cite{zugner2018adversarial} (direct/influence modes) and IG-Attack~\cite{wu2019adversarial} (3 integration steps); continuous-budget baselines are PGD-40~\cite{xu2019topology} and FGM (the $\ell_2$ analog of FGSM~\cite{goodfellow2014explaining}). Transfer targets are GAT~\cite{velivckovic2017graph}, GraphSAGE~\cite{hamilton2017inductive}, APPNP~\cite{gasteiger2018predict}, and JKNet~\cite{xu2018representation}. SepAttack uses $\pm\sqrt{\tilde{d}}$ paired with $\mathbf{w}_1$ as in Theorem~\ref{thm:separation}; both signs are tested and the more effective one retained (best-sign selection: two forward passes, still gradient-free), since ReLU breaks the sign symmetry of the linearized optimum. Targeted experiments deploy the rank-one optimum of Proposition~\ref{prop:targeted}.

Sections~\ref{sec:targeted_eff}--\ref{sec:sepat_results} compare SepAttack against state-of-the-art baselines on targeted attack effectiveness, heterophilous benchmarks, cross-architecture transfer, and adversarial training; further validation, including the symmetric factor-knockout ablation, is in Appendix~\ref{app:rank_one_validation}.

\subsection{Targeted Attack Effectiveness}
\label{sec:targeted_eff}

\begin{table}[t]
\centering
\caption{%
  Targeted feature attack comparison: overall success rate (\%)
  on the preselected target set (40~correctly-classified test
  nodes per dataset), averaged over 5~seeds.
  $\gamma$: features perturbed per target.
  Bold: best per column within each dataset.%
}
\label{tab:targeted}
\renewcommand{\arraystretch}{1.15}
\footnotesize
\begin{tabular}{@{}l l ccccc@{}}
\toprule
& & \multicolumn{5}{c}{\textbf{Success Rate (\%)}} \\
\cmidrule(lr){3-7}
\textbf{Method / Mode}
  & \textbf{Time}
  & $\gamma\!=\!1$ & $\gamma\!=\!2$ & $\gamma\!=\!3$ & $\gamma\!=\!4$ & $\gamma\!=\!5$ \\
\midrule
\multicolumn{7}{l}{\textit{Cora}} \\
\addlinespace[2pt]
\rowcolor{ourscolor}
SepAttack (our attack) / Direct    & 5\,ms
  & \best{13.0} & \best{17.5} & \best{20.0} & \best{22.0} & \best{23.0} \\
\rowcolor{ourscolor}
SepAttack / Influence & 6\,ms
  & 9.5 & 13.0 & 16.0 & 18.5 & 19.0 \\
Nettack / Direct     & 352\,ms
  & 8.0 & 15.0 & 17.5 & 18.5 & 19.5 \\
Nettack / Influence  & 407\,ms
  & 6.0 & 10.0 & 12.0 & 16.0 & 16.5 \\
IG-Attack / Direct   & 154\,s
  & 3.5 & 6.0 & 10.0 & 12.5 & 13.0 \\
\addlinespace[4pt]
\multicolumn{7}{l}{\textit{CiteSeer}} \\
\addlinespace[2pt]
\rowcolor{ourscolor}
SepAttack / Direct    & 12\,ms
  & \best{8.5} & \best{11.0} & \best{14.0} & \best{17.5} & \best{19.5} \\
\rowcolor{ourscolor}
SepAttack / Influence & 13\,ms
  & 7.5 & 9.0 & 11.5 & 13.0 & 15.5 \\
Nettack / Direct     & 417\,ms
  & 5.5 & 8.0 & 8.5 & 9.5 & 10.5 \\
Nettack / Influence  & 489\,ms
  & 5.0 & 7.5 & 8.5 & 9.5 & 10.0 \\
IG-Attack / Direct   & 524\,s
  & 2.0 & 3.0 & 4.5 & 5.0 & 6.0 \\
\bottomrule
\end{tabular}

\vspace{2pt}
{\scriptsize
Times are wall-clock per target node.
SepAttack achieves the highest success rate in all
20~comparison cells (2~datasets $\times$ 2~modes $\times$ 5~budgets).
}
\end{table}

Table~\ref{tab:targeted} compares SepAttack against
Nettack~\cite{zugner2018adversarial} and IG-Attack~\cite{wu2019adversarial}
in the \emph{targeted} setting: the attacker selects a single target node $t$ and aims to flip its prediction by maximizing the per-target logit gap $z_{t,c'} - z_{t,c^\ast}$ between the runner-up class $c'$ and the true class $c^\ast$, the objective of Nettack and IG-Attack. We derive the rank-one optimum for this targeted objective in Appendix~\ref{app:targeted_proof} (Proposition~\ref{prop:targeted}): the separation structure persists but specializes --- the graph factor becomes the $L$-hop influence column $\hat{\mathbf{A}}^L[\mathcal{A},t]$, encoding how strongly each node in the attacker set $\mathcal{A}$ propagates to $t$, and the weight factor becomes the logit-gap direction $\mathbf{W}(\mathbf{e}_{c'} - \mathbf{e}_{c^\ast})$. The integer-flip attack deployed in Table~\ref{tab:targeted} is its greedy discretization. Following Nettack's protocol, both \emph{direct} ($\mathcal{A} = \{t\}$) and \emph{influence} ($\mathcal{A} = $ neighbors of $t$, target itself untouched) modes are reported.
SepAttack achieves the highest success rate in all 20~cells
(2~datasets $\times$ 2~modes $\times$ 5~budgets): on Cora at
$\gamma\!=\!5$ in direct mode, it reaches $23.0\%$ versus $19.5\%$
for Nettack and $13.0\%$ for IG-Attack.
SepAttack requires approximately $5$\,ms per target
versus ${\sim}370$\,ms for Nettack, a ${\sim}70\times$ speedup
with zero gradient computation.

\subsection{Attack Effectiveness on Heterophilous Benchmarks}
\label{sec:heterophilous}

We stress-test the framework on three heterophilous graphs from the WebKB collection: Texas, Wisconsin, and Cornell~\citep{pei2020geomgcn}. These graphs lie outside the regime in which Theorem~\ref{thm:separation}'s rank-one approximation is tightest: their edge homophily is below $0.3$ and their normalized adjacency has spectral radius $\lambda_1{\approx}2.23\text{--}2.30$, well above the $\lambda_1{\approx}1.0$ regime of Cora, CiteSeer, and PubMed. We train a 2-layer GCN on the standard WebKB 60/20/20 split (index~0)~\citep{zhu2020beyond,chien2021gprgnn} and report misclassification at $\delta{=}10$ in Table~\ref{tab:heterophilous}, averaged over 5 seeds.

\begin{table}[t]
\centering
\caption{\textbf{Heterophilous benchmarks.} Misclassification rate (\%) at full attack budget $\delta{=}10$ on three WebKB heterophilous graphs~\citep{pei2020geomgcn}. Mean~$\pm$~std over 5 seeds; standard 60/20/20 split (index~0)~\citep{zhu2020beyond,chien2021gprgnn} (different from the 10/10/80 Nettack protocol of Section~\ref{sec:results}). PGD-40 follows the Madry~\citep{madry2018pgd} L2 convention (5 random restarts, $\alpha{=}2\delta/T$, normalized step); FGM is the L2 variant. The normalized-adjacency spectral radius is $\lambda_1{\approx}2.23$ on Wisconsin and $\lambda_1{\approx}2.30$ on Texas/Cornell, vs.\ $\lambda_1{\approx}1.0$ on the citation networks.}
\label{tab:heterophilous}
\small
\begin{tabular}{l r r r r}
\toprule
\textbf{Dataset} & \textbf{Clean mis.} & \textbf{FGM} & \textbf{PGD-40} & \textbf{SepAttack} \\
\midrule
Texas     & 38.4 & 81.6 $\pm$ 22.3 & 82.7 $\pm$ 22.3 & \best{89.2 $\pm$ \phantom{0}2.7} \\
Wisconsin & 47.1 & 64.7 $\pm$ 12.8 & 65.9 $\pm$ 10.2 & \best{73.7 $\pm$ \phantom{0}2.6} \\
Cornell   & 59.5 & 67.6 $\pm$ \phantom{0}2.7 & 65.4 $\pm$ \phantom{0}3.5 & \best{79.5 $\pm$ \phantom{0}6.5} \\
\bottomrule
\end{tabular}
\end{table}

SepAttack achieves the highest \emph{mean} misclassification on all three datasets, with margins of $+6.5$, $+7.8$, and $+11.9$\,pp over the best gradient baseline. A distinctive property of the closed-form construction is its low variance: SepAttack's standard deviation is $2.6$--$6.5$\,pp, compared with $10.2$--$22.3$\,pp for PGD-40 on Texas and Wisconsin and up to $22.3$\,pp for FGM on Texas. This is structural: the rank-one perturbation $\Delta\mathbf{X}^{\star}{=}\delta\,\mathbf{v}_1\mathbf{w}_1^\top$ is determined by graph structure and learned weights, not by gradient-based search, so it does not depend on initialization, step schedule, or restart luck. Per-seed, SepAttack beats both gradient baselines on $5/5$ seeds on Cornell; on Texas and Wisconsin it wins in mean while gradient methods exhibit large per-seed variability. We interpret this through Theorem~\ref{thm:separation}: on citation networks the spectrum concentrates near $\lambda_1{\approx}1$ and the graph factor $\sqrt{\tilde{d}_i}$ aligns the rank-one perturbation with the dominant propagation direction; on WebKB, $\lambda_1$ is more than double, so the rank-one geometry is a less complete summary of the attack surface --- yet the closed-form construction remains effective and is markedly more reliable than gradient baselines. Throughout these results the target model remains a 2-layer GCN; testing against heterophily-aware architectures (H2GCN~\citep{zhu2020beyond}, GPR-GNN~\citep{chien2021gprgnn}, LINKX~\citep{lim2021large}) is left to future work.

\subsection{Architecture-Independent Transfer}
\label{sec:transfer}

\begin{figure}[t]
    \centering
    \includegraphics[width=\textwidth]{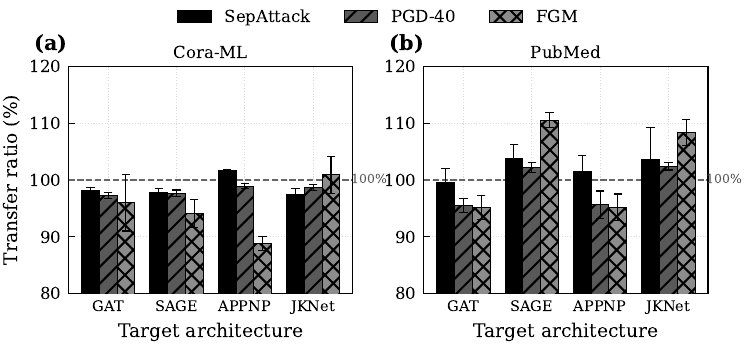}
    \caption{Cross-architecture transfer on continuous-feature graphs at $\delta=10$, 5 seeds. Bars: transfer ratio (\%) = accuracy drop on target / drop on source GCN; dashed line = perfect transfer. GCN source omitted (trivially $100\%$). Source drops (pp) used as denominators: Cora-ML: SepAttack=57.4, PGD=76.2, FGM=50.0; PubMed: SepAttack=43.7, PGD=64.1, FGM=47.8. All three attacks transfer within $6$\,pp of perfect efficiency, confirming the principle's architecture-invariance prediction when $\mathbf{w}_1$ is broadly shared. Binary-feature case (where $\mathbf{w}_1$ is not shared): Table~\ref{tab:binary_transfer}.}
    \label{fig:transfer}
\end{figure}

The separation principle decomposes $\Delta\mathbf{X}^*$ into a
graph factor $\tilde{\mathbf{D}}^{1/2}\mathbf{1}_N$
(architecture-independent by construction, since it depends only on
node degrees) and a weight factor $\mathbf{w}_1$ (from the source
GCN's weights). On continuous-feature graphs, where $\mathbf{w}_1$
is broadly shared across architectures trained on the same data,
the separation principle therefore predicts near-perfect transfer.
Figure~\ref{fig:transfer} tests this by crafting perturbations
against a source GCN and evaluating on GAT, GraphSAGE, APPNP, and
JKNet without adaptation.

On Cora-ML (Figure~\ref{fig:transfer}a), SepAttack transfer ratios
are 97.5--101.6\% (accuracy drops on non-source architectures
55.99--58.34\,pp versus 57.41\,pp on the source GCN). PubMed shows
the same pattern (99.5--103.8\%). PGD-40 and FGM transfer similarly
on both datasets (Figure~\ref{fig:transfer}), confirming that the
separation principle's architecture-invariance prediction holds for
all rank-one and gradient-based attacks when $\mathbf{w}_1$ is
shared across architectures.

On binary-feature graphs, $\mathbf{w}_1$ is not shared across architectures and SepAttack transfer drops to 21--69\% (Appendix~\ref{app:binary_transfer}, Table~\ref{tab:binary_transfer}); gradient-based attacks (PGD-40, FGM), which re-optimize per step, partially escape this limitation. Our separation principle offers an analytical analogue of the shared-subspace picture for GNNs documented in image classification~\cite{goodfellow2014explaining,tramer2017space}: the graph-dependent factor is architecture-invariant by construction, and transfer succeeds precisely to the extent that the weight-dependent factor $\mathbf{w}_1$ remains shared.

\subsection{SepAT: Adversarial Training Results}
\label{sec:sepat_results}

Table~\ref{tab:spectral_at} reports worst-case accuracy under best-sign SepAttack at $\delta=10$ for GCNs trained with and without SepAT, alongside each model's clean-accuracy baseline. SepAT recovers a substantial fraction of the accuracy lost to SepAttack on binary-feature graphs: worst-case accuracy improves from $16.0\%$ to $41.3\%$ on Cora ($+25.3$\,pp at $\delta=10$, average $+23.4$\,pp across $\delta\in\{2,4,6,8,10\}$) and from $15.4\%$ to $33.2\%$ on CiteSeer ($+17.8$\,pp at $\delta=10$, average $+15.1$\,pp). Clean accuracy is preserved within $0.5$ and $1.7$\,pp respectively.

\begin{table}[t]
\centering
\caption{SepAT (separation-principle adversarial training): worst-case accuracy (\%) under best-sign SepAttack at $\delta=10$. \textbf{Clean (no def.)}: test accuracy of the standard GCN without attack; \textbf{Clean (SepAT)}: test accuracy of the SepAT-regularized GCN without attack; \textbf{Worst-case}: $100 - \max(\mathrm{err}(+\sqrt{\tilde{d}},\mathbf{w}_1),\, \mathrm{err}(-\sqrt{\tilde{d}},\mathbf{w}_1))$ at $\delta = 10$. \textbf{Avg.\ gain}: mean reduction in worst-case error across $\delta \in \{2,4,6,8,10\}$, in percentage points. SepAT uses the bilateral output-stability regularizer of Section~\ref{sec:robust_training} ($\alpha{=}0.25$, $\delta_{\max}{=}10$, $T_w{=}50$, $\lambda{=}5\times10^{-4}$); Planetoid split~\cite{kipf2016semi}, 2-layer GCN with 16 hidden units; see Appendix~\ref{app:bilateral}.}
\label{tab:spectral_at}
\renewcommand{\arraystretch}{1.2}
\footnotesize
\begin{tabular}{@{}l c c c c c@{}}
\toprule
\textbf{Dataset} & \textbf{Clean (no def.)} & \textbf{Clean (SepAT)} & \textbf{Worst-case (no def.)} & \textbf{Worst-case (SepAT)} & \textbf{Avg.\ gain} \\
\midrule
  Cora     & 82.7 & 82.2 & 16.0 & \best{41.3} & $+23.4$\,pp \\
  CiteSeer & 69.9 & 68.2 & 15.4 & \best{33.2} & $+15.1$\,pp \\
  PubMed   & 78.1 & 75.9 & 40.6 & \best{42.6} & $+1.6$\,pp \\
\bottomrule
\end{tabular}
\end{table}

On PubMed the gain is small ($+1.6$\,pp average), in line with the ablation of Appendix~\ref{app:rank_one_validation}: on continuous-feature graphs, the weight factor $\mathbf{w}_1$ alone accounts for only ${\sim}25\%$ of the total damage from SepAttack, with the remaining ${\sim}75\%$ contributed by the multiplicative spatial amplification. SepAT regularizes the model output along the single $\mathbf{w}_1$ direction, addressing the $\mathbf{w}_1$ component but not the multiplicative spatial part; extending the defense to regularize along a broader range of spectral directions is a natural next step.


\section{Conclusion}
\label{sec:conclusion}

We presented a graph-theoretic framework that derives adversarial feature attacks on GNNs from the spectral properties of $\hat{\mathbf{A}}$. Theorem~\ref{thm:separation} establishes a \emph{separation principle}: the optimal feature perturbation under full access is rank-one and factorizes into a degree-proportional graph-spectral direction and an independent weight-spectral direction, yielding a gradient-free attack with analytical guarantees. Empirically, SepAttack matches or beats Nettack and IG-Attack with up to ${\sim}70\times$ speedup, transfers near-perfectly across architectures on continuous-feature graphs, and SepAT improves worst-case accuracy at $\delta=10$ by $+25.3$\,pp on Cora and $+17.8$\,pp on CiteSeer (Section~\ref{sec:sepat_results}); the smaller PubMed gain ($+1.6$\,pp average) is consistent with the ablation finding (Appendix~\ref{app:rank_one_validation}).

\paragraph{Limitations.} The separation principle is derived under linearized ReLU; restricted-access (Proposition~\ref{prop:restricted}) is heuristic, not provably optimal; transfer degrades to 21--69\% on binary-feature graphs (Appendix~\ref{app:binary_transfer}); the framework is restricted to feature perturbations, with edge perturbations and graph-transformer extensions left open; SepAT regularizes along $\pm\Delta\mathbf{X}^*$ only and is not evaluated against adaptive attacks targeting other spectral directions. Future directions include a nonlinear analysis, spectral robustness certificates, and edge-perturbation analogues.

\paragraph{Broader Impact.} This work reveals that GNNs in security-sensitive domains are vulnerable to gradient-free attacks that existing defenses fail to address. We characterize these attacks to enable principled defense design; SepAT and the open-problem framing of multi-subspace robustness (Section~\ref{sec:sepat_results}) are direct consequences. We acknowledge that the closed-form attack lowers the computational barrier for adversarial perturbation, and mitigate this by releasing the defense alongside the attack code.

\bibliographystyle{plainnat}
\bibliography{references}


\appendix

\section{Extended Related Work and Comparison Table}
\label{app:lit_comparison}

This appendix expands the brief Related Work of Section~\ref{sec:related_work} into a full categorization by attack strategy.

\paragraph{The two attack surfaces on a GNN.} Adversarial robustness studies how a learned model behaves when its input is corrupted by a small, adversarially chosen perturbation. For graph neural networks, the input has two components --- the graph $\mathbf{A}$ and the node features $\mathbf{X}$ --- and the literature has accordingly developed two complementary families of attacks: \emph{feature attacks}, which perturb $\mathbf{X}$~\cite{zugner2018adversarial,wu2019adversarial,xu2019topology}, and \emph{structure attacks}, which add, remove, or reweight edges in $\mathbf{A}$~\cite{zugner2018adversarial,xu2019topology,zugner2019adversarial,dai2018adversarial}. Each family further branches into evasion (test-time) and poisoning (training-time) variants.

\paragraph{Gradient-based attacks.} The dominant paradigm exploits gradient information to identify influential perturbations. Nettack~\cite{zugner2018adversarial} pioneered targeted attacks using greedy perturbation selection guided by surrogate model scores while preserving degree distributions. FGA~\cite{chen2018fast} adapted FGSM to graphs by selecting edge flips with maximum gradient magnitude. \citet{xu2019topology} introduced PGD attacks that relax binary adjacency constraints to continuous optimization. IG-Attack~\cite{wu2019adversarial} applies integrated gradients to score feature/edge perturbations. FGSM~\cite{goodfellow2014explaining} remains a single-step gradient baseline widely used in feature-attack settings. More recently, GraD~\cite{liu2023towards} addressed gradient bias in untargeted attacks, and InfAttack~\cite{wang2024efficient} leveraged influence functions for efficient black-box approximations. These methods fundamentally depend on model-specific gradient signals rather than intrinsic graph properties.

\paragraph{Meta-learning attacks.} For poisoning scenarios, Metattack~\cite{zugner2019adversarial} introduced meta-gradients to optimize perturbations as hyperparameters, achieving effective global degradation. Metacon~\cite{yoon2024debiased} addressed the bias toward training nodes through contrastive surrogate losses. While powerful, these methods remain computationally intensive and gradient-dependent.

\paragraph{Reinforcement learning attacks.} RL methods address black-box scenarios where gradient access is unavailable. RL-S2V~\cite{dai2018adversarial} modeled attacks as Markov Decision Processes with Structure2Vec embeddings, learning policies from prediction feedback alone. \citet{ma2019attacking} extended this to edge rewiring that preserves edge counts. G2A2C~\cite{ju2023let} introduced gradient-free node injection via actor-critic methods. These approaches learn attack strategies empirically but lack theoretical guarantees on perturbation effectiveness.

\paragraph{Node injection attacks.} Recognizing that modifying existing edges is often impractical, recent work has shifted toward node injection. NIPA~\cite{sun2020adversarial} formulated hierarchical RL for injection, while TDGIA~\cite{zou2021tdgia} introduced topologically-aware injection with smooth feature optimization. HAO~\cite{chen2022understanding} improved unnoticeability through homophily preservation. WTGIA~\cite{lei2024intruding} extended attacks to text-attributed graphs through word frequency analysis. These methods optimize injection heuristically without exploiting structural graph invariants.

\paragraph{Topological-heuristic attacks.} A limited body of work incorporates graph topological heuristics into attack design. GF-Attack~\cite{chang2020restricted} used graph signal processing for restricted black-box scenarios. VoteAttack~\cite{wen2024voteattack} derived a node voting matrix from graph topology for gradient-free attacks. GOttack~\cite{alom2025gottack} identified critical graph orbits (equivalence classes within graphlets) as attack targets. These approaches employ spectral information as auxiliary heuristics rather than as the foundational framework for attack construction.

\paragraph{Spectral perspectives on GNN robustness.} A spectral view of adversarial robustness on graphs is well-developed on the \emph{defense} side. GCN-SVD~\cite{entezari2020all} demonstrated empirically that a rank-$10$ truncated SVD of the adjacency matrix removes most of the damage caused by Nettack-style perturbations, indicating that adversarial perturbations concentrate in the high-rank components and therefore that clean graphs are spectrally low-rank. ProGNN~\cite{jin2020graph} made this explicit by jointly learning a cleaned adjacency via nuclear-norm, sparsity, and feature-smoothness regularization. \citet{chang2021not} strengthened the picture theoretically: using matrix perturbation theory, they prove that only spectral components of the normalized Laplacian whose eigenvalues fall within a specific robust interval are stable under structural perturbation, and derive the GCN-LFR defense that filters out the unstable modes. On the \emph{attack} side, the spectral line is sparser. \citet{bojchevski2019adversarial} pioneered the use of eigenvalue perturbation theory for adversarial attacks, but their analysis targets DeepWalk node embeddings rather than GCN classification and their attack is greedy rather than closed-form. SPAC~\cite{lin2022graph} introduced a structural attack that maximizes the $\ell_2$ distance between the eigenvalues of the clean and perturbed normalized Laplacians via iterative projected gradient descent. SAT~\cite{li2022spectral} decomposes the normalized adjacency spectrally and performs adversarial training by perturbing eigenvalues and eigenvectors via gradient ascent. In every case, the optimization remains iterative and the attack is constructed implicitly through a gradient procedure.

\paragraph{Certified robustness and adversarial training for GNNs.} A complementary line of work develops certified robustness guarantees for GNNs via randomized smoothing~\cite{bojchevski2019certifiable}. \citet{mujkanovic2022defenses} showed that many widely cited GNN defenses lose most of their apparent robustness under adaptive attacks. \citet{gosch2023adversarial} diagnose two fundamental pitfalls of adversarial training on GNNs---transductive evaluation bias and the limited capacity of standard GCN architectures for robust representations---and show that learnable spectral graph diffusion models are needed to improve upon the PGD-based adversarial training baseline. SepAT derives its training objective directly from the closed-form worst-case attack of Theorem~\ref{thm:separation}, complementing the PGD-based baseline of Gosch et al.

Table~\ref{tab:lit_comparison} extends the categorization to a broader set of attack methods, classifying each by strategy, knowledge assumption, perturbation type, attack stage, and theoretical basis.

\begin{table*}[t]
\centering
\caption{Comparison of adversarial attack methods on GNNs for node classification (Time Ascending). Full references in Section~\ref{sec:related_work}.}
\label{tab:lit_comparison}
\renewcommand{\arraystretch}{1.3}
\footnotesize
\begin{tabular}{@{}l c c c c c l@{}}
\toprule
\textbf{Method} & \textbf{Year} & \textbf{Strategy} & \textbf{Knowledge} & \textbf{Perturbation} & \textbf{Stage} & \textbf{Theoretical Basis} \\
\midrule
Nettack & 2018 & Gradient & Grey-box & Struct+Feat & Both & Surrogate scoring \\
\addlinespace[2pt]
FGA & 2018 & Gradient & White-box & Structure & Evasion & FGSM adaptation \\
\addlinespace[2pt]
RL-S2V & 2018 & RL & Black-box & Structure & Both & Learned policy \\
\addlinespace[2pt]
PGD Attack & 2019 & Gradient & White-box & Structure & Both & Continuous relaxation \\
\addlinespace[2pt]
Metattack & 2019 & Meta-learning & Grey-box & Structure & Poisoning & Meta-gradients \\
\addlinespace[2pt]
TDGIA & 2021 & Optimization & Grey-box & Injection & Evasion & Smooth optimization \\
\addlinespace[2pt]
HAO & 2022 & Optimization & Grey-box & Injection & Both & Homophily heuristic \\
\addlinespace[2pt]
GraD & 2022 & Gradient & White-box & Structure & Poisoning & Gradient debiasing \\
\addlinespace[2pt]
G2A2C & 2023 & RL & Black-box & Injection & Evasion & Actor-critic \\
\addlinespace[2pt]
VoteAttack & 2024 & Heuristic & Black-box & Struct+Feat & Evasion & Voting matrix \\
\addlinespace[2pt]
GOttack & 2025 & Topological & Grey-box & Structure & Evasion & Graph orbits \\
\midrule
\textbf{Ours} &  & \textbf{Graph-theoretic} & \textbf{White-box} & \textbf{Feature} & \textbf{Both} & \textbf{Spectral decomposition} \\
\bottomrule
\end{tabular}
\end{table*}

\section{Two-Layer GCN Specialization}
\label{app:two_layer}

\paragraph{Justification for the two-layer choice.} The choice of a two-layer GCN is not arbitrary but reflects a fundamental limitation of deep graph neural networks. \citet{li2018deeper} demonstrated that the GCN propagation rule is mathematically equivalent to Laplacian smoothing: each layer computes a weighted average of neighboring features, causing node representations within connected components to converge toward a common value proportional to $\sqrt{\tilde{d}_v}$. While moderate smoothing (1--2 layers) is beneficial for leveraging graph structure, excessive smoothing leads to \emph{over-smoothing}, where node features become indistinguishable, directly contradicting the discrimination required for node classification. \citet{oono2019graph} formalized this observation by proving that GNN outputs converge exponentially fast to an eigenspace of the graph Laplacian, with the convergence rate governed by the spectral properties of the normalized adjacency matrix and the singular values of the weight matrices. In this limiting eigenspace, node representations encode only connected component membership and degree information, which is insufficient for classification. Empirical studies consistently confirm that two-layer GCNs achieve optimal performance on standard benchmarks, with accuracy degrading significantly beyond three layers~\cite{kipf2016semi, li2018deeper}.

\subsection{Residual Analysis Framework}
\label{app:residual}

This subsection derives the closed form of the feature residual after linearizing the ReLU activations of~\eqref{eq:gcn_layerwise}, the surrogate construction adopted in Section~\ref{sec:feature_attack}. We replace each ReLU activation by the identity and trace the perturbation through the resulting linear $L$-layer GCN.

Consider an input feature matrix $\mathbf{X}$ and an adversarial perturbation $\Delta\mathbf{X}$ that produces the modified feature matrix $\mathbf{X}_a = \mathbf{X} + \Delta\mathbf{X}$. For the first layer of the linearized GCN:
\begin{equation}
\label{eq:app_first_layer}
    \mathbf{H}^{(1)} = \hat{\mathbf{A}}(\mathbf{X} + \Delta\mathbf{X})\mathbf{W}^{(0)} = \hat{\mathbf{A}}\mathbf{X}\mathbf{W}^{(0)} + \underbrace{\hat{\mathbf{A}}\Delta\mathbf{X}\mathbf{W}^{(0)}}_{\mathbf{R}_f^{(1)}(\Delta\mathbf{X})}.
\end{equation}
By induction over layers, the perturbation propagates through the linearized $L$-layer GCN as $\mathbf{R}_f^{(\ell+1)}(\Delta\mathbf{X}) = \hat{\mathbf{A}}\,\mathbf{R}_f^{(\ell)}(\Delta\mathbf{X})\,\mathbf{W}^{(\ell)}$ for $\ell = 1,\ldots,L-1$, which after telescoping yields the $L$-th layer feature residual:
\begin{equation}
\label{eq:app_feature_residual}
    \mathbf{R}_f^{(L)}(\Delta\mathbf{X}) = \hat{\mathbf{A}}^L\,\Delta\mathbf{X}\,\mathbf{W},
\end{equation}
where $\mathbf{W} = \mathbf{W}^{(0)}\mathbf{W}^{(1)}\cdots\mathbf{W}^{(L-1)} \in \mathbb{R}^{D\times C}$ is the composed weight matrix. The perturbation effect is propagated through the $L$-hop neighborhood structure encoded in $\hat{\mathbf{A}}^L$. For the two-layer GCN ($L=2$) used in our experiments, this specializes to $\mathbf{R}_f^{(2)}(\Delta\mathbf{X}) = \hat{\mathbf{A}}^2\Delta\mathbf{X}\,\mathbf{W}^{(0)}\mathbf{W}^{(1)}$, with $\hat{\mathbf{A}}^2 = \tilde{\mathbf{D}}^{-1/2}\tilde{\mathbf{A}}\tilde{\mathbf{D}}^{-1}\tilde{\mathbf{A}}\tilde{\mathbf{D}}^{-1/2}$ encoding two-hop propagation.

\subsection{Complete Optimization Formulation}
\label{app:optimization}

This appendix details the chain of reformulations that connects Problem~\ref{prob:nonlinear} (worst-case feature attack on a nonlinear GCN) to the surrogate problem~\eqref{eq:relaxed_problem} solved in Section~\ref{sec:feature_attack}: a ReLU linearization that simplifies the objective, and the equivalence $\|\Delta\mathbf{X}\|_F^2 = \|\Delta\mathbf{X}\|_0$ that justifies replacing the discrete feature-flip constraint by a Frobenius-norm budget on binary-feature graphs.

\paragraph{Detailed combinatorial form of Problem~\ref{prob:nonlinear}.}
The adversarial feature attack involves two coupled decisions: (1) how to allocate the perturbation budget across nodes, and (2) which features to perturb at each node. Making the per-node budget structure explicit, Problem~\ref{prob:nonlinear} is equivalent to
\begin{align}
\label{eq:app_feature_level}
\max_{\Delta\mathbf{X}} \quad & \bigl\|\mathbf{H}^{(L)}(\mathbf{X}+\Delta\mathbf{X}) - \mathbf{H}^{(L)}(\mathbf{X})\bigr\|_2 \\
\text{s.t.} \quad & \|\Delta\mathbf{X}_{u,:}\|_0 \leq \gamma_u \quad \forall u \in \mathcal{V}, \nonumber \\
& \Delta\mathbf{X} \in \{-1, 0, 1\}^{N \times D}, \nonumber \\
& \mathbf{X} + \Delta\mathbf{X} \in \{0,1\}^{N \times D}, \nonumber \\
& \sum_{u \in \mathcal{V}} \gamma_u \leq \Gamma_{\text{total}}, \nonumber
\end{align}
where $\gamma_u$ is the per-node perturbation budget (number of features that may be flipped at node $u$) and $\Gamma_{\text{total}}$ is the total perturbation budget. The objective is the nonlinear residual norm; the feasible set is combinatorial, so~\eqref{eq:app_feature_level} is intractable in this raw form.

\begin{remark}[Implicit Domain of $\Delta\mathbf{X}$]
\label{remark:implicit_domain}
The constraint $\mathbf{X} + \Delta\mathbf{X} \in \{0,1\}^{N \times D}$ combined with $\mathbf{X} \in \{0,1\}^{N \times D}$ implies $\Delta\mathbf{X} \in \{-1, 0, 1\}^{N \times D}$. Specifically, if $\mathbf{X}_{ij} = 0$, then $\Delta\mathbf{X}_{ij} \in \{0, 1\}$; if $\mathbf{X}_{ij} = 1$, then $\Delta\mathbf{X}_{ij} \in \{-1, 0\}$.
\end{remark}

\paragraph{Linearization of the ReLU.}
Replacing each ReLU in the GCN forward map~\eqref{eq:gcn_layerwise} by the identity, the residual~\eqref{eq:residual_def} reduces to the closed form $\mathbf{R}_f^{(L)}(\Delta\mathbf{X}) = \hat{\mathbf{A}}^L\,\Delta\mathbf{X}\,\mathbf{W}$ (Appendix~\ref{app:residual}), with $\mathbf{W}=\mathbf{W}^{(0)}\cdots\mathbf{W}^{(L-1)}$. Substituting this expression into Problem~\ref{prob:nonlinear} yields the \emph{linearized} mixed-integer feature-attack problem:
\begin{equation}
\label{eq:app_linearized_integer}
\begin{aligned}
\max_{\Delta \mathbf{X}} \quad & \|\hat{\mathbf{A}}^L \Delta \mathbf{X} \mathbf{W}\|_2 \\
\mathrm{s.t.} \quad & \|\Delta\mathbf{X}\|_0 \leq \Gamma_{\text{total}}, \\
& \mathbf{X} + \Delta\mathbf{X} \in \{0,1\}^{N \times D}.
\end{aligned}
\end{equation}
The objective is now closed-form in $\Delta\mathbf{X}$, but the feasible set remains combinatorial (a mixed-integer program).

\paragraph{From integer feature-flips to a Frobenius-norm budget.}
We replace the integer constraints of~\eqref{eq:app_linearized_integer} by a Frobenius-norm bound. By Remark~\ref{remark:implicit_domain}, every entry of any feasible $\Delta\mathbf{X}$ lies in $\{-1,0,1\}$, so each nonzero entry contributes exactly~$1$ to the squared Frobenius norm; summing,
\begin{equation}
\label{eq:app_frobenius_bound}
\|\Delta\mathbf{X}\|_F^2 \;=\; \|\Delta\mathbf{X}\|_0 \;\leq\; \Gamma_{\text{total}} \;\triangleq\; \gamma.
\end{equation}
Maximizing $\|\cdot\|_2$ over a compact set is equivalent to maximizing $\|\cdot\|_2^2$, and we obtain the relaxed problem~\eqref{eq:relaxed_problem} stated in the main text. This is a genuine relaxation along two axes: (i)~the ReLU is replaced by the identity, but the 1-Lipschitz property of ReLU ensures the linearized residual norm upper-bounds the nonlinear one; (ii)~every feasible point of~\eqref{eq:app_linearized_integer} is feasible for~\eqref{eq:relaxed_problem}, so the integer-flip feasible set is contained in the Frobenius ball. Together these imply that the optimal value of~\eqref{eq:relaxed_problem} provides an upper bound on those of~\eqref{eq:app_linearized_integer} and Problem~\ref{prob:nonlinear}. The relaxation trades discrete feature-flip constraints for a continuous budget, enabling the closed-form spectral analysis of Appendix~\ref{app:thm_proof}.

\paragraph{Greedy discretization of the rank-one optimum (restricted access).}
The continuous direction $\Delta\mathbf{X}^{\mathcal{A}} = \sqrt{\gamma}\,(\mathbf{b}_\mathcal{A}/\|\mathbf{b}_\mathcal{A}\|_2)\,\mathbf{w}_1^\top$ from Proposition~\ref{prop:restricted} also yields a feasible attack for the integer feature-flip Problem~\ref{prob:nonlinear}: the rank-one outer product $(\mathbf{b}_\mathcal{A}/\|\mathbf{b}_\mathcal{A}\|_2)\,\mathbf{w}_1^\top$ produces a per-coordinate score
\[
    s_{ij} \;\propto\; [\mathbf{b}_\mathcal{A}]_i \cdot [\mathbf{w}_1]_j,
\]
and applying greedy top-$\Gamma_{\text{total}}$ selection of \emph{feasible} flips (i.e., $0\to 1$ where $\mathbf{X}_{ij}=0$ and $1\to 0$ where $\mathbf{X}_{ij}=1$) according to $\mathrm{sign}(s_{ij})$ produces a discrete $\Delta\mathbf{X}\in\{-1,0,1\}^{N\times D}$ supported on $\mathcal{A}$ that satisfies all constraints of Problem~\ref{prob:nonlinear}.

\section{Proof of Theorem~\ref{thm:separation} (Separation Principle)}
\label{app:thm_proof}

This appendix gives the complete proof of Theorem~\ref{thm:separation}. Starting from the linear cost, continuous feature attack problem~\eqref{eq:relaxed_problem}, we (i)~rewrite the problem in the eigenbasis of $\hat{\mathbf{A}}^{2L}$, reducing it to an optimization over coefficient vectors $\{\boldsymbol{\epsilon}_i\}_{i=1}^N$; (ii)~solve that reformulated problem by a Weyl/Rayleigh sandwich argument, identifying the optimal coefficients as $\boldsymbol{\epsilon}_1^\star=\sqrt{\gamma}\,\mathbf{w}_1$, $\boldsymbol{\epsilon}_i^\star=\mathbf{0}$ ($i\geq 2$) and the optimal value as $\lambda_1\mu_1\gamma$; (iii)~identify the dominant eigenpair of $\hat{\mathbf{A}}^{2L}$ in closed form via the normalized Laplacian, yielding $\mu_1=1$ and $\mathbf{v}_1 = \tilde{\mathbf{D}}^{1/2}\mathbf{1}_N/\|\tilde{\mathbf{D}}^{1/2}\mathbf{1}_N\|_2$; and (iv)~translate the optimal coefficients back to the original perturbation space, yielding the rank-one separation~\eqref{eq:optimal_DX}.

The integer-to-Frobenius relaxation that produces~\eqref{eq:relaxed_problem} from Problem~\ref{prob:nonlinear} is given in Appendix~\ref{app:optimization}.

\subsection{Reformulation in the Spectral Basis}
\label{app:thm_reformulation}

\paragraph{Properties of $\hat{\mathbf{A}}^{2L}$.}
The matrix $\hat{\mathbf{A}}^{2L}$ that drives the analysis is symmetric and positive semi-definite. Symmetry follows from $\hat{\mathbf{A}}^T = \hat{\mathbf{A}}$; positive semi-definiteness follows from $\hat{\mathbf{A}}^{2L} = (\hat{\mathbf{A}}^L)^T (\hat{\mathbf{A}}^L)$. Consequently all eigenvalues of $\hat{\mathbf{A}}^{2L}$ satisfy $\mu_i \geq 0$, and the eigenvectors form an orthonormal basis of $\mathbb{R}^N$.

\paragraph{Eigendecompositions.}
Let $(\mu_1, \mathbf{v}_1), \ldots, (\mu_N, \mathbf{v}_N)$ be the eigenvalue--eigenvector pairs of $\hat{\mathbf{A}}^{2L}$, with $\mu_1 \geq \mu_2 \geq \cdots \geq \mu_N \geq 0$ and $\{\mathbf{v}_i\}$ orthonormal. Likewise let $(\lambda_1, \mathbf{w}_1), \ldots, (\lambda_D, \mathbf{w}_D)$ be the pairs of $\mathbf{W}\mathbf{W}^T \in \mathbb{R}^{D \times D}$, with $\lambda_1 \geq \lambda_2 \geq \cdots \geq \lambda_D \geq 0$ and $\{\mathbf{w}_j\}$ orthonormal.

\paragraph{Perturbation expansion.}
Since $\{\mathbf{v}_1, \ldots, \mathbf{v}_N\}$ is an orthonormal basis of $\mathbb{R}^N$, every $\Delta\mathbf{X} \in \mathbb{R}^{N \times D}$ admits a unique expansion
\begin{equation}
\label{eq:app_DX_expansion}
    \Delta\mathbf{X} \;=\; \sum_{i=1}^{N} \mathbf{v}_i\, \boldsymbol{\epsilon}_i^T,
    \qquad \boldsymbol{\epsilon}_i \in \mathbb{R}^D,
\end{equation}
where $\boldsymbol{\epsilon}_i$ collects the contribution of the $i$-th eigenmode to the perturbation.

\paragraph{Cost-function transformation.}
The spectral norm satisfies $\|\mathbf{C}\|_2 = \sqrt{\lambda_{\max}(\mathbf{C}^T\mathbf{C})}$, so maximizing $\|\hat{\mathbf{A}}^L\,\Delta\mathbf{X}\,\mathbf{W}\|_2^2$ is equivalent to maximizing $\lambda_{\max}(\mathbf{C}^T\mathbf{C})$ for $\mathbf{C} = \hat{\mathbf{A}}^L\,\Delta\mathbf{X}\,\mathbf{W}$. Using $(\hat{\mathbf{A}}^L)^T \hat{\mathbf{A}}^L = \hat{\mathbf{A}}^{2L}$, problem~\eqref{eq:relaxed_problem} becomes
\begin{equation}
\label{eq:app_opt_lmax}
\max_{\Delta\mathbf{X}}\; \lambda_{\max}\!\left(\mathbf{W}^T\,\Delta\mathbf{X}^T\,\hat{\mathbf{A}}^{2L}\,\Delta\mathbf{X}\,\mathbf{W}\right) \quad \text{s.t.} \quad \|\Delta\mathbf{X}\|_F^2 \leq \gamma.
\end{equation}
The objective matrix is symmetric and positive semi-definite for every $\Delta\mathbf{X}$, since it equals $\mathbf{B}^T \mathbf{B}$ with $\mathbf{B} = \hat{\mathbf{A}}^L\,\Delta\mathbf{X}\,\mathbf{W}$; this $\mathbf{W}^T(\cdot)\mathbf{W}$ form will be preserved throughout, so that Weyl's inequality and PSD monotonicity apply.

Substituting the expansion~\eqref{eq:app_DX_expansion} into~\eqref{eq:app_opt_lmax} and using the eigenvalue identity $\hat{\mathbf{A}}^{2L}\mathbf{v}_i = \mu_i \mathbf{v}_i$ together with orthonormality $\mathbf{v}_i^T\mathbf{v}_j = \delta_{ij}$:
\begin{align}
\mathbf{W}^T\,\Delta\mathbf{X}^T\,\hat{\mathbf{A}}^{2L}\,\Delta\mathbf{X}\,\mathbf{W}
&= \mathbf{W}^T \!\left(\sum_{j=1}^{N} \boldsymbol{\epsilon}_j \mathbf{v}_j^T\right) \!\!\left(\sum_{i=1}^{N} \mu_i\, \mathbf{v}_i \boldsymbol{\epsilon}_i^T\right) \!\mathbf{W} \nonumber \\
&= \mathbf{W}^T \!\left(\sum_{i,j} \mu_i\, \boldsymbol{\epsilon}_j (\mathbf{v}_j^T \mathbf{v}_i) \boldsymbol{\epsilon}_i^T\right) \!\mathbf{W} \nonumber \\
&= \mathbf{W}^T \!\left(\sum_{i=1}^{N} \mu_i\, \boldsymbol{\epsilon}_i \boldsymbol{\epsilon}_i^T\right) \!\mathbf{W} \nonumber \\
&= \sum_{i=1}^{N} \mu_i\, \mathbf{W}^T \boldsymbol{\epsilon}_i \boldsymbol{\epsilon}_i^T \mathbf{W}.
\label{eq:app_cost_final}
\end{align}
Each summand $\mu_i\,\mathbf{W}^T\boldsymbol{\epsilon}_i\boldsymbol{\epsilon}_i^T\mathbf{W}$ is symmetric and positive semi-definite (since $\mu_i \geq 0$ and $\mathbf{W}^T\boldsymbol{\epsilon}_i\boldsymbol{\epsilon}_i^T\mathbf{W} = (\boldsymbol{\epsilon}_i^T\mathbf{W})^T(\boldsymbol{\epsilon}_i^T\mathbf{W}) \succeq \mathbf{0}$).

\paragraph{Constraint simplification.}
By orthonormality of $\{\mathbf{v}_i\}$,
\begin{align}
\|\Delta\mathbf{X}\|_F^2
&= \operatorname{trace}(\Delta\mathbf{X}^T \Delta\mathbf{X}) \nonumber\\
&= \operatorname{trace}\!\biggl(\sum_{i,j}\boldsymbol{\epsilon}_j(\mathbf{v}_j^T\mathbf{v}_i)\boldsymbol{\epsilon}_i^T\biggr)
= \sum_{i=1}^{N}\|\boldsymbol{\epsilon}_i\|_2^2,
\label{eq:app_constraint_transformed}
\end{align}
so the budget constraint~\eqref{eq:relaxed_problem} reads $\sum_i\|\boldsymbol{\epsilon}_i\|_2^2 \leq \gamma$.

\paragraph{Reformulated problem.}
Combining~\eqref{eq:app_cost_final} and~\eqref{eq:app_constraint_transformed}, problem~\eqref{eq:relaxed_problem} is equivalent to
\begin{equation}
\label{eq:app_reformulated_final}
\begin{aligned}
    &\max_{\{\boldsymbol{\epsilon}_i\}_{i=1}^N} \quad \lambda_{\max}\!\left(\sum_{i=1}^{N} \mu_i\, \mathbf{W}^T \boldsymbol{\epsilon}_i \boldsymbol{\epsilon}_i^T \mathbf{W}\right) \\[6pt]
    &\text{s.t.}\quad \text{(C1):}\; \sum_{i=1}^{N}\|\boldsymbol{\epsilon}_i\|_2^2 \;\leq\; \gamma,
\end{aligned}
\end{equation}
which is the same as~\eqref{eq:reformulated_final} in the main text. Optimizing over $\Delta\mathbf{X}$ has been replaced by optimizing over the coefficient vectors $\{\boldsymbol{\epsilon}_i\}_{i=1}^N$.

\subsection{Optimal Spectral Coefficients}
\label{app:thm_lemma}

We solve the reformulated problem~\eqref{eq:app_reformulated_final} via a sandwich argument that combines Weyl's inequality from above with a single-summand restriction from below.

\begin{lemma}[Optimal Spectral Coefficients]
\label{lem:optimal}
Let $(\mu_1,\mathbf{v}_1)$ be the dominant eigenpair of $\hat{\mathbf{A}}^{2L}$ and $(\lambda_1,\mathbf{w}_1)$ the dominant eigenpair of $\mathbf{W}\mathbf{W}^T$, with $\|\mathbf{v}_1\|_2 = \|\mathbf{w}_1\|_2 = 1$. An optimal solution to~\eqref{eq:app_reformulated_final} is
\begin{equation}
\label{eq:app_optimal_coeffs}
    \boldsymbol{\epsilon}_1^\star = \sqrt{\gamma}\cdot\mathbf{w}_1,
    \qquad \boldsymbol{\epsilon}_2^\star = \cdots = \boldsymbol{\epsilon}_N^\star = \mathbf{0},
\end{equation}
achieving the optimal value $\lambda_1\mu_1\gamma$.
\end{lemma}

\begin{proof}
We establish $\lambda_1\mu_1\gamma$ as both an upper bound and an attained value.

\medskip
\noindent\textbf{Step 1: Sandwich chain inequality.}
Let $\boldsymbol{\epsilon}_1,\ldots,\boldsymbol{\epsilon}_N$ be any feasible point, i.e., satisfying the budget $\sum_i\|\boldsymbol{\epsilon}_i\|_2^2 \leq \gamma$ (constraint~C1).

\emph{Lower bound.} Each summand $\mu_i\,\mathbf{W}^T\boldsymbol{\epsilon}_i\boldsymbol{\epsilon}_i^T\mathbf{W}$ is symmetric and positive semi-definite, so dropping all summands except $i=1$ can only decrease $\lambda_{\max}$:
\begin{equation}
\label{eq:app_lower_bound}
\lambda_{\max}\!\left(\mu_1\,\mathbf{W}^T\boldsymbol{\epsilon}_1\boldsymbol{\epsilon}_1^T\mathbf{W}\right) \;\leq\; \lambda_{\max}\!\left(\sum_{i=1}^{N} \mu_i\,\mathbf{W}^T\boldsymbol{\epsilon}_i\boldsymbol{\epsilon}_i^T\mathbf{W}\right).
\end{equation}

\emph{Upper bound (Weyl's inequality).} For symmetric PSD matrices, Weyl's inequality gives
\begin{equation}
\label{eq:app_weyl_upper}
\lambda_{\max}\!\left(\sum_{i=1}^{N} \mu_i\,\mathbf{W}^T\boldsymbol{\epsilon}_i\boldsymbol{\epsilon}_i^T\mathbf{W}\right) \;\leq\; \sum_{i=1}^{N} \mu_i\,\lambda_{\max}\!\left(\mathbf{W}^T\boldsymbol{\epsilon}_i\boldsymbol{\epsilon}_i^T\mathbf{W}\right).
\end{equation}

For each $i$, $\mathbf{W}^T\boldsymbol{\epsilon}_i\boldsymbol{\epsilon}_i^T\mathbf{W}$ is symmetric of rank at most one, so its unique nonzero eigenvalue equals its trace:
\begin{equation}
\lambda_{\max}\!\left(\mathbf{W}^T\boldsymbol{\epsilon}_i\boldsymbol{\epsilon}_i^T\mathbf{W}\right) = \boldsymbol{\epsilon}_i^T\mathbf{W}\mathbf{W}^T\boldsymbol{\epsilon}_i = \|\mathbf{W}^T\boldsymbol{\epsilon}_i\|_2^2.
\end{equation}
Combining,
\begin{equation}
\label{eq:app_full_chain}
\mu_1\,\|\mathbf{W}^T\boldsymbol{\epsilon}_1\|_2^2 \;\leq\; \lambda_{\max}\!\left(\sum_{i=1}^{N}\mu_i\,\mathbf{W}^T\boldsymbol{\epsilon}_i\boldsymbol{\epsilon}_i^T\mathbf{W}\right) \;\leq\; \sum_{i=1}^{N}\mu_i\,\|\mathbf{W}^T\boldsymbol{\epsilon}_i\|_2^2.
\end{equation}

\medskip
\noindent\textbf{Step 2: Maximize the lower bound over C1.}
The Rayleigh quotient $\boldsymbol{\epsilon}_1^T\mathbf{W}\mathbf{W}^T\boldsymbol{\epsilon}_1 / \|\boldsymbol{\epsilon}_1\|_2^2$ is maximized at $\boldsymbol{\epsilon}_1 \propto \mathbf{w}_1$ with value $\lambda_1$. Hence, putting the entire budget on the first coefficient ($\boldsymbol{\epsilon}_1 = \sqrt{\gamma}\,\mathbf{w}_1$, $\boldsymbol{\epsilon}_i = \mathbf{0}$ for $i\geq 2$),
\begin{equation}
\label{eq:app_lower_max}
\max_{\boldsymbol{\epsilon}_i \in \text{C1}} \mu_1\,\|\mathbf{W}^T\boldsymbol{\epsilon}_1\|_2^2 \;=\; \mu_1\cdot\lambda_1\gamma.
\end{equation}

\medskip
\noindent\textbf{Step 3: Bound the upper bound over C1.}
For each $i$, the Rayleigh bound gives $\|\mathbf{W}^T\boldsymbol{\epsilon}_i\|_2^2 = \boldsymbol{\epsilon}_i^T\mathbf{W}\mathbf{W}^T\boldsymbol{\epsilon}_i \leq \lambda_1\,\|\boldsymbol{\epsilon}_i\|_2^2$. Combined with $\mu_i \leq \mu_1$ and the budget,
\begin{equation}
\label{eq:app_rhs_chain}
\sum_{i=1}^{N}\mu_i\,\|\mathbf{W}^T\boldsymbol{\epsilon}_i\|_2^2 \;\leq\; \lambda_1\sum_{i=1}^{N}\mu_i\,\|\boldsymbol{\epsilon}_i\|_2^2 \;\leq\; \lambda_1\mu_1\sum_{i=1}^{N}\|\boldsymbol{\epsilon}_i\|_2^2 \;\leq\; \lambda_1\mu_1\gamma.
\end{equation}

\medskip
\noindent\textbf{Step 4: Sandwich.}
Substituting~\eqref{eq:app_lower_max} into the maximum of the chain~\eqref{eq:app_full_chain} from below and~\eqref{eq:app_rhs_chain} from above:
\begin{equation}
\lambda_1\mu_1\gamma \;\leq\; \max_{\boldsymbol{\epsilon}_i \in \text{C1}} \lambda_{\max}\!\left(\sum_{i=1}^{N}\mu_i\,\mathbf{W}^T\boldsymbol{\epsilon}_i\boldsymbol{\epsilon}_i^T\mathbf{W}\right) \;\leq\; \lambda_1\mu_1\gamma.
\end{equation}
Therefore the optimal value of~\eqref{eq:app_reformulated_final} is exactly $\lambda_1\mu_1\gamma$.

\medskip
\noindent\textbf{Step 5: Achievability.}
The candidate~\eqref{eq:app_optimal_coeffs} is feasible:
\(\sum_{i=1}^{N}\|\boldsymbol{\epsilon}_i^\star\|_2^2 = \|\boldsymbol{\epsilon}_1^\star\|_2^2 = \gamma\,\|\mathbf{w}_1\|_2^2 = \gamma\).
And it attains the bound:
\begin{align}
\lambda_{\max}\!\left(\sum_{i=1}^N \mu_i\,\mathbf{W}^T\boldsymbol{\epsilon}_i^\star(\boldsymbol{\epsilon}_i^\star)^T\mathbf{W}\right)
&= \lambda_{\max}\!\left(\mu_1\,\mathbf{W}^T\boldsymbol{\epsilon}_1^\star(\boldsymbol{\epsilon}_1^\star)^T\mathbf{W}\right) \nonumber\\
&= \mu_1\,\|\mathbf{W}^T\boldsymbol{\epsilon}_1^\star\|_2^2 \nonumber\\
&= \mu_1 \cdot \gamma\,\mathbf{w}_1^T\mathbf{W}\mathbf{W}^T\mathbf{w}_1 \nonumber\\
&= \mu_1\lambda_1\gamma.
\end{align}
\end{proof}

\paragraph{Geometric remark.}
The optimum collapses onto a single eigenmode rather than spreading across several. By the Weyl PSD inequality $\lambda_{\max}(\sum_i A_i) \leq \sum_i \lambda_{\max}(A_i)$, distributing budget across multiple $\boldsymbol{\epsilon}_i$ can only weaken the objective relative to concentrating the budget into the top spectral mode. Equality in this inequality requires the $A_i$ to share a common top eigenvector, which (combined with $\mu_1 \geq \mu_i$) forces concentration on $i=1$.

\subsection{Closed Form of the Dominant Eigenpair via the Normalized Laplacian}
\label{app:dominant_eigenpair}

It remains to identify the dominant eigenpair $(\mu_1,\mathbf{v}_1)$ of $\hat{\mathbf{A}}^{2L}$ in closed form.

\paragraph{Spectrum of $\hat{\mathbf{A}}$.}
Since $\hat{\mathbf{A}} = \tilde{\mathbf{D}}^{-1/2}\tilde{\mathbf{A}}\tilde{\mathbf{D}}^{-1/2}$ is the symmetric normalized adjacency of the graph with self-loops, we may write $\hat{\mathbf{A}} = \mathbf{I} - \hat{\mathbf{L}}$ where $\hat{\mathbf{L}}$ is the symmetric normalized Laplacian, whose eigenvalues lie in $[0,2]$ and whose smallest eigenvalue is $0$. We assume the graph is connected; for the citation networks we use, we follow the standard convention of restricting to the largest connected component when relevant. Under this assumption the zero eigenvalue of $\hat{\mathbf{L}}$ is simple. Adding self-loops makes $\hat{\mathbf{A}}$ aperiodic, so the eigenvalue $-1$ of $\hat{\mathbf{A}}$ (which would otherwise appear for bipartite graphs) is excluded; consequently the largest eigenvalue $1$ of $\hat{\mathbf{A}}$ is the unique eigenvalue of maximum modulus.

\paragraph{Dominant eigenpair of $\hat{\mathbf{A}}$.}
The smallest eigenvalue of $\hat{\mathbf{L}}$ is $0$, with eigenvector $\mathbf{u}_0 = \tilde{\mathbf{D}}^{1/2}\mathbf{1}_N / \|\tilde{\mathbf{D}}^{1/2}\mathbf{1}_N\|_2$:
\begin{align}
\hat{\mathbf{L}}\,\tilde{\mathbf{D}}^{1/2}\mathbf{1}_N
&= (\mathbf{I} - \tilde{\mathbf{D}}^{-1/2}\tilde{\mathbf{A}}\tilde{\mathbf{D}}^{-1/2})\,\tilde{\mathbf{D}}^{1/2}\mathbf{1}_N \nonumber\\
&= \tilde{\mathbf{D}}^{1/2}\mathbf{1}_N - \tilde{\mathbf{D}}^{-1/2}\tilde{\mathbf{A}}\mathbf{1}_N \nonumber\\
&= \tilde{\mathbf{D}}^{1/2}\mathbf{1}_N - \tilde{\mathbf{D}}^{-1/2}\tilde{\mathbf{D}}\mathbf{1}_N \nonumber\\
&= \tilde{\mathbf{D}}^{1/2}\mathbf{1}_N - \tilde{\mathbf{D}}^{1/2}\mathbf{1}_N = \mathbf{0},
\end{align}
where we used $\tilde{\mathbf{A}}\mathbf{1}_N = \tilde{\mathbf{D}}\mathbf{1}_N$ (each row of $\tilde{\mathbf{A}}$ sums to the corresponding entry of $\tilde{\mathbf{D}}$). Consequently the largest eigenvalue of $\hat{\mathbf{A}} = \mathbf{I} - \hat{\mathbf{L}}$ is $1$, attained at the same eigenvector $\mathbf{u}_0$.

\paragraph{Dominant eigenpair of $\hat{\mathbf{A}}^{2L}$.}
Powers of a symmetric matrix share its eigenvectors: if $\hat{\mathbf{A}}\mathbf{u}_i = \eta_i\,\mathbf{u}_i$, then $\hat{\mathbf{A}}^{2L}\mathbf{u}_i = \eta_i^{2L}\,\mathbf{u}_i$. Since $\eta_i \in [-1,1]$ with maximum absolute value $1$ attained at $\eta_0 = 1$, and since $2L$ is even, we have $\eta_i^{2L} \in [0,1]$ with maximum $1$. Therefore
\begin{equation}
\label{eq:app_top_eigenpair}
\mu_1 \;=\; 1, \qquad \mathbf{v}_1 \;=\; \mathbf{u}_0 \;=\; \frac{\tilde{\mathbf{D}}^{1/2}\mathbf{1}_N}{\|\tilde{\mathbf{D}}^{1/2}\mathbf{1}_N\|_2}.
\end{equation}
In particular, the $i$-th entry of $\mathbf{v}_1$ is proportional to $\sqrt{\tilde{d}_i}$, and $\mu_1 = 1$ is independent of the depth $L$ --- the dominant eigenpair is the same for $\hat{\mathbf{A}}$, $\hat{\mathbf{A}}^2$, and any other power. The optimal value $\lambda_1\mu_1\gamma$ from Lemma~\ref{lem:optimal} therefore reduces to $\lambda_1\gamma$.

\subsection{Translation Back: Closed-Form Optimal Perturbation}
\label{app:thm_translation}

We can now complete the proof of Theorem~\ref{thm:separation}. Using the optimal coefficients $\boldsymbol{\epsilon}_1^\star = \sqrt{\gamma}\,\mathbf{w}_1$, $\boldsymbol{\epsilon}_i^\star = \mathbf{0}$ ($i\geq 2$) of Lemma~\ref{lem:optimal} and the eigenbasis expansion~\eqref{eq:app_DX_expansion},
\begin{equation}
\label{eq:app_optimal_DX_v1}
\Delta\mathbf{X}^\star \;=\; \sum_{i=1}^N \mathbf{v}_i\,(\boldsymbol{\epsilon}_i^\star)^T \;=\; \mathbf{v}_1\,(\boldsymbol{\epsilon}_1^\star)^T \;=\; \sqrt{\gamma}\,\mathbf{v}_1\,\mathbf{w}_1^T.
\end{equation}
Substituting the closed form $\mathbf{v}_1 = \tilde{\mathbf{D}}^{1/2}\mathbf{1}_N/\|\tilde{\mathbf{D}}^{1/2}\mathbf{1}_N\|_2$ from~\eqref{eq:app_top_eigenpair},
\begin{equation}
\Delta\mathbf{X}^\star \;=\; \sqrt{\gamma}\,\frac{\tilde{\mathbf{D}}^{1/2}\mathbf{1}_N}{\|\tilde{\mathbf{D}}^{1/2}\mathbf{1}_N\|_2}\,\mathbf{w}_1^T,
\end{equation}
which is exactly~\eqref{eq:optimal_DX} in the body. The per-node perturbation magnitude is $\|\Delta\mathbf{X}^\star_{i,:}\|_2 \propto \sqrt{\tilde{d}_i}$, so the optimal attack places budget on each node in proportion to the square root of its (self-loop-augmented) degree. This proves Theorem~\ref{thm:separation}.

\paragraph{Rank-one structure and role of the two factors.}
The expression $\Delta\mathbf{X}^\star = \sqrt{\gamma}\,\mathbf{v}_1\,\mathbf{w}_1^T$ has rank exactly one. The graph-dependent factor $\mathbf{v}_1 = \tilde{\mathbf{D}}^{1/2}\mathbf{1}_N/\|\cdot\|$ determines \emph{which nodes to perturb}: the $i$-th entry, proportional to $\sqrt{\tilde{d}_i}$, controls how much of the $\sqrt{\gamma}$ budget is placed on node $i$, and depends only on the graph degrees. The weight-dependent factor $\mathbf{w}_1$ determines \emph{which feature direction to perturb along}: the $j$-th entry controls how much of the per-node perturbation lies along feature $j$, and depends only on the learned weights through the top right singular vector of $\mathbf{W} = \mathbf{W}^{(0)}\cdots\mathbf{W}^{(L-1)}$. The two factors are computed independently: $\mathbf{v}_1$ from $\tilde{\mathbf{D}}$ alone (graph), $\mathbf{w}_1$ from $\mathbf{W}\mathbf{W}^T$ alone (weights).

\subsection{Extended Discussion of Theorem~\ref{thm:separation}}
\label{app:thm_discussion}

This subsection collects the extended discussion of Theorem~\ref{thm:separation} that was abridged in the body for space. It comprises two parts: the three constructive consequences of the separation principle (Appendix~\ref{app:thm_discussion_principle}), and the computational comparison between SepAttack and gradient-based baselines (Appendix~\ref{app:thm_discussion_compute}).

\subsubsection{Separation principle: the three consequences}
\label{app:thm_discussion_principle}

The rank-one structure of \eqref{eq:optimal_DX} exposes the \emph{separation principle}: $\tilde{\mathbf{D}}^{1/2}\mathbf{1}_N$ determines node importance (per-node magnitude $\propto\sqrt{\tilde{d}_i}$), while $\mathbf{w}_1$ determines feature importance, and the two factors are computed independently. This factorization has three consequences. (i)~The optimum admits a closed-form expression computed from a single eigendecomposition; in particular, problem~\eqref{eq:relaxed_problem} need not be solved by a gradient-based procedure as in PGD~\cite{xu2019topology}, IG-Attack~\cite{wu2019adversarial}, or Nettack~\cite{zugner2018adversarial}. (ii)~The optimal value $\lambda_1\gamma$ depends on the weight spectrum only through $\lambda_1$, since $\mu_1=1$ is fixed by the normalization of $\hat{\mathbf{A}}$. (iii)~The graph-dependent factor depends only on $\tilde{\mathbf{D}}$, which is shared across architectures trained on the same graph; consequently the perturbation transfers to a different architecture to the extent that its top weight singular vector $\mathbf{w}_1$ is aligned with the source GCN's. A further immediate consequence is a theoretical perspective on the empirical degree bias reported by GraD~\cite{liu2023towards}, since the principle shows the degree-proportional allocation arises naturally from the optimum rather than from gradient-estimation artifacts.

\subsubsection{Computational aspects}
\label{app:thm_discussion_compute}

A direct consequence of the closed-form characterization is that crafting the optimal perturbation is dramatically cheaper than for gradient-based attacks. Computing $\mathbf{w}_1$ amounts to a single eigendecomposition of $\mathbf{W}\mathbf{W}^T \in \mathbb{R}^{D\times D}$ (or, equivalently, the top-singular-vector computation on $\mathbf{W}$), which is $O(DC)$ per power-iteration step and converges in a handful of steps in practice; the graph-dependent factor $\tilde{\mathbf{D}}^{1/2}\mathbf{1}_N$ is a constant of the graph and is computed once. By contrast, gradient-based baselines repeatedly call the GCN's forward and backward passes: PGD-40 requires $40$ forward/backward passes per attack; IG-Attack additionally averages gradients over $K$ integration steps, costing $O(K)$ extra forward/backward passes; Nettack performs greedy combinatorial search over candidate edges or features. Empirically (Section~\ref{sec:targeted_eff}, Table~\ref{tab:targeted}), this translates to ${\sim}5$\,ms per target for SepAttack versus ${\sim}370$\,ms for Nettack and ${\sim}154$\,s for IG-Attack on Cora --- a $70\times$ speedup over Nettack and a four-orders-of-magnitude reduction relative to IG-Attack.

\section{Targeted Rank-One Optimum}
\label{app:targeted_proof}

This appendix states and proves the targeted specialization of the separation principle. Unlike Theorem~\ref{thm:separation}, which maximizes the global residual norm $\|\hat{\mathbf{A}}^L\Delta\mathbf{X}\,\mathbf{W}\|_2$, here the attacker focuses on a single target node $t$ and maximizes a per-target logit-gap objective. This is the objective adopted by Nettack~\cite{zugner2018adversarial} (their Eq.~14) and IG-Attack~\cite{wu2019adversarial}, and we develop the closed-form rank-one optimum for it so that Section~\ref{sec:targeted_eff} can compare against those baselines on their own objective.

\paragraph{Connection to Nettack's surrogate loss.}
After linearization, Nettack's surrogate loss $\mathcal{L}_s(A,X;W,v_0)$ for target $v_0$ with true class $c_{\text{old}}$ and runner-up class $c$ reduces (their Eq.~14, $L=2$) to
\[
    [\hat{\mathbf{A}}^2 \mathbf{X}\mathbf{W}]_{v_0,c} - [\hat{\mathbf{A}}^2 \mathbf{X}\mathbf{W}]_{v_0,c_{\text{old}}}.
\]
For an additive feature perturbation $\Delta\mathbf{X}$, the change in this loss is precisely $[\hat{\mathbf{A}}^2 \Delta\mathbf{X}\,\mathbf{W}]_{v_0,c}-[\hat{\mathbf{A}}^2 \Delta\mathbf{X}\,\mathbf{W}]_{v_0,c_{\text{old}}}$, which is the linearized targeted objective of Proposition~\ref{prop:targeted} below for $L=2$. Nettack maximizes this objective by greedy combinatorial search; Proposition~\ref{prop:targeted} solves the same objective in closed form under a rank-one Frobenius-norm constraint, and the integer-flip attack we deploy in Table~\ref{tab:targeted} is the greedy discretization of that closed-form rank-one optimum.

\begin{proposition}[Targeted Rank-One Optimum]
\label{prop:targeted}
Fix a target node $t$ with true class $c^\ast$ and runner-up class $c'$, and an attacker set $\mathcal{A}\subseteq\mathcal{V}$. Consider the linearized targeted objective
\begin{equation}
\label{eq:app_targeted_obj}
    \phi(\Delta\mathbf{X}) \;:=\; \bigl[\hat{\mathbf{A}}^L\Delta\mathbf{X}\,\mathbf{W}\bigr]_{t,c'} \;-\; \bigl[\hat{\mathbf{A}}^L\Delta\mathbf{X}\,\mathbf{W}\bigr]_{t,c^\ast}
\end{equation}
under a rank-one perturbation $\Delta\mathbf{X}=\alpha\,\mathbf{v}\mathbf{w}^\top$ with $\|\mathbf{v}\|_2=\|\mathbf{w}\|_2=1$, support constraint $\mathrm{supp}(\mathbf{v})\subseteq\mathcal{A}$, and Frobenius budget $\|\Delta\mathbf{X}\|_F^2\le\gamma$. The unique (up to sign) optimum factorizes as
\begin{equation}
\label{eq:targeted_heuristic}
    \mathbf{v}^\star \;\propto\; \hat{\mathbf{A}}^L[\mathcal{A},\,t], \qquad \mathbf{w}^\star \;\propto\; \mathbf{W}\bigl(\mathbf{e}_{c'}-\mathbf{e}_{c^\ast}\bigr), \qquad \alpha^\star = \sqrt{\gamma},
\end{equation}
where $\mathbf{W}=\mathbf{W}^{(0)}\mathbf{W}^{(1)}\cdots\mathbf{W}^{(L-1)}$ is the composed weight matrix of the linearized $L$-layer GCN, and $\hat{\mathbf{A}}^L[\mathcal{A},t]$ denotes the column of $\hat{\mathbf{A}}^L$ indexed by $t$, restricted to rows in $\mathcal{A}$ and zero-padded elsewhere.
\end{proposition}

\begin{proof}
Substituting $\Delta\mathbf{X} = \alpha\,\mathbf{v}\mathbf{w}^\top$ into~\eqref{eq:app_targeted_obj} and using $[\hat{\mathbf{A}}^L\Delta\mathbf{X}\,\mathbf{W}]_{t,c} = \alpha\cdot(\hat{\mathbf{A}}^L\mathbf{v})_t \cdot (\mathbf{w}^\top\mathbf{W})_c = \alpha\cdot\langle \hat{\mathbf{A}}^L[\mathcal{A},t],\mathbf{v}\rangle\cdot\langle \mathbf{W}^\top\mathbf{w},\mathbf{e}_c\rangle$ (since $\mathbf{v}$ is supported on $\mathcal{A}$, only entries of $\hat{\mathbf{A}}^L[\cdot,t]$ in rows $\mathcal{A}$ contribute), the objective separates as
\[
    \phi(\Delta\mathbf{X}) \;=\; \alpha\cdot\langle\hat{\mathbf{A}}^L[\mathcal{A},t],\,\mathbf{v}\rangle\cdot\langle\mathbf{W}(\mathbf{e}_{c'}-\mathbf{e}_{c^\ast}),\,\mathbf{w}\rangle.
\]
Each inner product is independently maximized by aligning the unit vector with its argument: $\mathbf{v}^\star = \hat{\mathbf{A}}^L[\mathcal{A},t]/\|\hat{\mathbf{A}}^L[\mathcal{A},t]\|_2$ (with the support restriction automatically satisfied) and $\mathbf{w}^\star = \mathbf{W}(\mathbf{e}_{c'}-\mathbf{e}_{c^\ast})/\|\mathbf{W}(\mathbf{e}_{c'}-\mathbf{e}_{c^\ast})\|_2$. The Frobenius budget $\|\alpha\,\mathbf{v}\mathbf{w}^\top\|_F = |\alpha|$ binds at $\alpha^\star = \sqrt{\gamma}$. Uniqueness up to sign follows from the Cauchy--Schwarz equality conditions.
\end{proof}

\paragraph{Discussion.}
Proposition~\ref{prop:targeted} preserves the graph--weight separation of Theorem~\ref{thm:separation} but with both factors specialized to the targeted regime. The graph factor becomes the $L$-hop influence column $\hat{\mathbf{A}}^L[\mathcal{A},t]$ instead of the global degree-proportional allocation $\sqrt{\tilde{\mathbf{d}}}$, because in the targeted setting only nodes in $\mathcal{A}$ that propagate to $t$ within $L$ hops can affect $t$'s prediction --- a node outside $t$'s $L$-hop neighborhood has zero entry in this column. Likewise, the weight factor becomes the per-target logit-gap direction $\mathbf{W}(\mathbf{e}_{c'}-\mathbf{e}_{c^\ast})$ instead of the global Frobenius-optimal singular vector $\mathbf{w}_1$, because the targeted objective only cares about the difference between two specific output coordinates. For the two-layer GCN ($L=2$) used in our experiments, $\hat{\mathbf{A}}^L[\mathcal{A},t]$ is the standard two-hop influence column and $\mathbf{W}=\mathbf{W}^{(0)}\mathbf{W}^{(1)}$.

\paragraph{Greedy discretization for integer feature flips.}
The continuous rank-one optimum $\Delta\mathbf{X}^\star = \sqrt{\gamma}\,\mathbf{v}^\star(\mathbf{w}^\star)^\top$ is converted to an integer-flip attack by the same greedy top-$\gamma$ procedure described in Appendix~\ref{app:optimization}, with $\mathbf{v}^\star$ and $\mathbf{w}^\star$ playing the roles of $\mathbf{b}_\mathcal{A}/\|\mathbf{b}_\mathcal{A}\|$ and $\mathbf{w}_1$ respectively, and per-target budget $\gamma$ replacing the global budget $\Gamma_{\text{total}}$. This is the integer-flip attack deployed in Table~\ref{tab:targeted}.

\section{SepAT Derivation}
\label{app:defense_derivation}

This appendix derives the SepAT objective of Section~\ref{sec:robust_training}. Appendix~\ref{app:naive_penalty} shows why a na\"ive spectral-norm penalty on $\mathbf{W}$ fails. Appendix~\ref{app:bilateral} derives the bilateral output-stability regularizer. Appendix~\ref{app:curriculum} specifies the curriculum schedule.

\subsection{Why the Na\"ive Penalty Fails}
\label{app:naive_penalty}

A natural first attempt at hardening the GCN against SepAttack would be to penalize the spectral norm $\sigma_{\max}(\mathbf{W})$ of the composed weight matrix, since Theorem~\ref{thm:separation} shows the attack magnitude depends on $\lambda_1(\mathbf{W}\mathbf{W}^T) = \sigma_{\max}(\mathbf{W})^2$. This approach is ineffective.

\paragraph{Scale invariance of the classification task.}
Under the rescaling $\mathbf{W} \mapsto \alpha\mathbf{W}$ with $\alpha > 0$:
\begin{itemize}
    \item The adversarial residual scales as $\|\hat{\mathbf{A}}^L\Delta\mathbf{X}^*\alpha\mathbf{W}\|_2 = \alpha\,\|\hat{\mathbf{A}}^L\Delta\mathbf{X}^*\mathbf{W}\|_2$.
    \item The clean classification margin between class $c$ and class $c'$ at node $u$, defined as $\mathbf{z}_{u,c} - \mathbf{z}_{u,c'} = [\hat{\mathbf{A}}^L\mathbf{X}\alpha\mathbf{W}]_{u,c} - [\hat{\mathbf{A}}^L\mathbf{X}\alpha\mathbf{W}]_{u,c'}$, scales as $\alpha\,(\mathbf{z}_{u,c}-\mathbf{z}_{u,c'})$.
\end{itemize}
Consequently, both adversarial residual and clean margins scale identically in $\alpha$, and the \emph{relative} vulnerability measured by the ratio residual/margin is invariant. Minimizing $\sigma_{\max}(\mathbf{W})$ merely rescales the entire logit geometry without improving robustness: the optimizer drives $\alpha \to 0$ to minimize the penalty, degrading both adversarial and clean performance proportionally.

\paragraph{Consequence.}
Robustness cannot be obtained by shrinking $\mathbf{W}$. It must be obtained by ensuring that the model's \emph{output} is stable under the worst-case perturbation, independently of the overall scale of the weights. This motivates the bilateral output-stability regularizer of Appendix~\ref{app:bilateral}.

\subsection{Bilateral Output-Stability Regularization}
\label{app:bilateral}

Given the current weights $\theta$, let $\Delta\mathbf{X}^* = \delta\cdot \mathbf{v}_1\,\mathbf{w}_1^T$ denote the full-access SepAttack perturbation of Theorem~\ref{thm:separation}, with $\mathbf{v}_1 = \tilde{\mathbf{D}}^{1/2}\mathbf{1}_N/\|\tilde{\mathbf{D}}^{1/2}\mathbf{1}_N\|_2$ and $\mathbf{w}_1$ the top left-singular vector of $\mathbf{W}^{(0)}\mathbf{W}^{(1)}$ at the current $\theta$. We require the model output to remain close to its clean value under both signs of the perturbation:
\begin{equation}
\label{eq:app_bilateral_penalty}
    \mathcal{R}_{\text{spec}}(\theta) \;=\; \frac{1}{|\mathcal{V}_L|}\sum_{u \in \mathcal{V}_L}\sum_{s\in\{+,-\}}\bigl\|f_\theta(\mathbf{A},\,\mathbf{X}+s\Delta\mathbf{X}^*)_u - f_\theta(\mathbf{A},\mathbf{X})_u\bigr\|_2^2.
\end{equation}

\paragraph{Why bilateral?}
ReLU is asymmetric: $\mathrm{ReLU}(x+\epsilon)$ and $\mathrm{ReLU}(x-\epsilon)$ are not mirror images of $\mathrm{ReLU}(x)$ at a general pre-activation $x$. In particular, at a pre-activation in the active regime ($x>0$), a positive perturbation $+\epsilon$ propagates linearly while a negative perturbation $-\epsilon$ may shut off the unit entirely. Thus the attack $\Delta\mathbf{X}^*$ and its sign-flipped counterpart $-\Delta\mathbf{X}^*$ elicit qualitatively different responses from the nonlinear GCN, and training against only one sign hardens the model in only one direction. Regularizing against both signs closes this asymmetry.

\paragraph{Full objective.}
The SepAT loss is
\begin{equation}
\label{eq:app_full_loss}
    \mathcal{L}_{\text{robust}}(\theta) \;=\; \mathcal{L}_{\text{CE}}(f_\theta(\mathbf{A},\mathbf{X}),\mathbf{y}) \;+\; \lambda\|\mathbf{W}^{(0)}\|_F^2 \;+\; \alpha\,\mathcal{R}_{\text{spec}}(\theta),
\end{equation}
with $\lambda$ the standard weight-decay coefficient and $\alpha$ the spectral regularizer weight.

\paragraph{Recomputing $\Delta\mathbf{X}^*$ at each step.}
Since $\mathbf{w}_1$ depends on the current weights, the SepAttack perturbation $\Delta\mathbf{X}^*$ must be recomputed at each training step. The graph-dependent factor $\mathbf{v}_1 = \tilde{\mathbf{D}}^{1/2}\mathbf{1}_N/\|\cdot\|$ is a constant of the graph and is computed once at the start of training. The feature-space direction $\mathbf{w}_1$ is the top singular vector of the $D\times C$ matrix $\mathbf{W}^{(0)}\mathbf{W}^{(1)}$, computed by power iteration; in the training loop, warm-starting from the previous step's $\mathbf{w}_1$ makes a single power-iteration update per step sufficient, with per-step cost $O(DC)$.

\subsection{Curriculum Schedule}
\label{app:curriculum}

Introducing the spectral regularizer at full strength from the start of training destabilizes optimization: at initialization the model has no meaningful structure, and $\mathcal{R}_{\text{spec}}$ penalizes changes to an uninformative reference. To avoid this, we ramp the perturbation budget linearly from $0$ to $\delta_{\max}$ over $T_w$ warmup epochs:
\begin{equation}
\label{eq:app_curriculum}
    \delta_t \;=\; \min\!\left(\frac{t}{T_w},\,1\right)\cdot\delta_{\max}, \qquad t = 0, 1, \ldots
\end{equation}
We use $T_w = 50$, $\delta_{\max} = 10$, $\alpha = 0.25$, and $\lambda = 5\times 10^{-4}$ throughout. With this schedule, the first $T_w$ epochs function as standard training with a small spectral regularization, and spectral hardening engages gradually as the classifier acquires meaningful structure. In practice we found this curriculum essential for stable optimization, removing it produced erratic training dynamics and degraded both clean and robust accuracy.

\section{Empirical Validation of the Closed-Form Rank-One Optimum}
\label{app:rank_one_validation}

Theorem~\ref{thm:separation} derives the closed-form rank-one perturbation $\Delta\mathbf{X}^\star = \sqrt{\gamma}\,\mathbf{v}_1\mathbf{w}_1^\top$ under the linearized GCN. This appendix verifies that the optimum retains its predicted dominance once the full ReLU non-linearity is reinstated, along two complementary lines: a head-to-head comparison against alternative rank-one attacks (Figure~\ref{fig:spectral-validation}), and a symmetric factor-knockout ablation that probes the multiplicative structure of the optimum (Table~\ref{tab:ablation}).

\paragraph{Comparison among rank-one attacks.}
Figure~\ref{fig:spectral-validation} compares four rank-one strategies sharing the same Frobenius budget on the nonlinear ReLU-GCN: \textbf{SepAttack} (deployed: $\sqrt{\tilde{d}}$ aligned with $\mathbf{w}_1$), \textbf{Uniform-$\mathbf{w}_1$} (uniform spatial allocation, retains $\mathbf{w}_1$), \textbf{Random} (both factors random unit vectors), and \textbf{Unif-Unif} (both factors uniform; sanity floor). In the full-access setting (panel a), SepAttack reaches the highest misclassification rate on Cora, CiteSeer, and Cora-ML at every budget; on PubMed, SepAttack leads at small-to-moderate budgets, while Unif-Unif catches up and slightly overtakes at the largest budget --- so SepAttack is the strongest among the rank-one separable attacks tested in most, but not all, scenarios. The gap to Uniform-$\mathbf{w}_1$ quantifies the graph-factor contribution at $3$--$10$\,pp at $\delta=10$. In the restricted-access setting (panel b, $|\mathcal{A}|{\approx}100$), the heuristic of Proposition~\ref{prop:restricted} again dominates its feasible-set peers. Per-seed variance is small at every budget; the closed-form construction does not depend on initialization or stochastic search.

\begin{figure}[t]
    \centering
    \includegraphics[width=\textwidth]{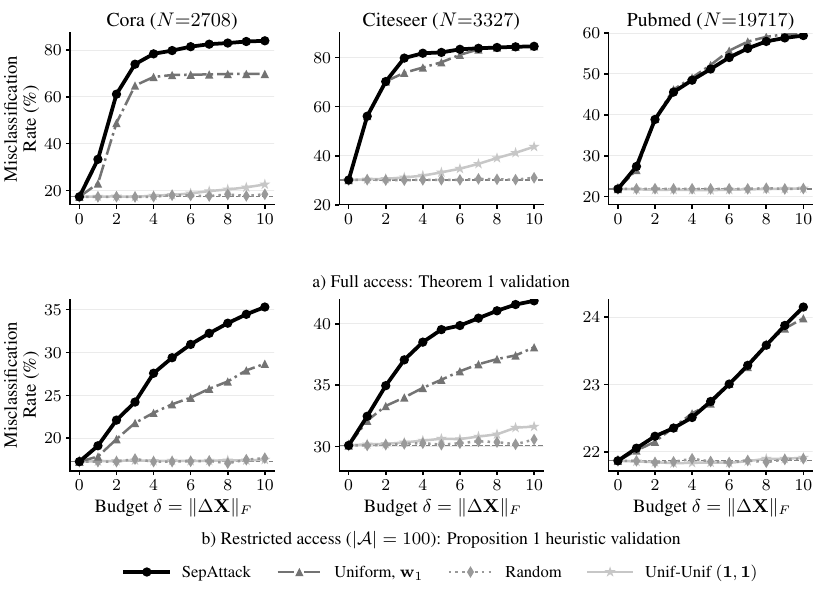}
    \caption{SepAttack on the nonlinear ReLU-GCN, 5 seeds. \textbf{(a)}~Full access: SepAttack (black) achieves the highest misclassification rate on Cora, CiteSeer, and Cora-ML at every budget; on PubMed, Unif-Unif overtakes at the largest budget. \textbf{(b)}~Restricted access ($|\mathcal{A}|{=}100$, attacker set drawn uniformly at random once per seed and matched across methods): the Proposition~\ref{prop:restricted} heuristic $\mathrm{SepAttack}|\mathcal{A}$ dominates feasible-set peers. Dashed grey: clean misclassification.}
    \label{fig:spectral-validation}
\end{figure}

\paragraph{Symmetric factor-knockout ablation.}
Theorem~\ref{thm:separation} predicts a \emph{multiplicative} rank-one
structure: the optimal perturbation is a product of a graph factor
$\sqrt{\tilde{\mathbf{d}}}/\|\sqrt{\tilde{\mathbf{d}}}\|$ and a weight
factor $\mathbf{w}_1$, and neither factor alone is predicted to be optimal --- the multiplicative form means that any attack lacking either the optimal graph factor or the optimal weight factor strictly underperforms the combined optimum. Figure~\ref{fig:spectral-validation} above tests one slice of this prediction (the optimal $\sqrt{\tilde d}$ allocation against uniform spatial allocation, both retaining $\mathbf{w}_1$); Table~\ref{tab:ablation} below tests it symmetrically, by independently randomizing each factor at the full-access budget $\delta\!=\!10$.

\begin{table}[t]
\centering
\caption{%
  Separation-principle ablation at $\delta=10$ (full access):
  accuracy (\%) under each attack component, averaged over 5~seeds.
  \textbf{Random}: $u_\text{rand}\cdot w_\text{rand}^T$, both unit vectors drawn iid from
  $\mathcal{N}(0,I)$ and renormalized (floor baseline).
  \textbf{Graph-only}: $u_\text{best}\cdot w_\text{rand}^T$ (optimal spatial allocation, random
  weight direction).
  \textbf{Weight-only}: $u_\text{rand}\cdot\mathbf{w}_1^T$ (random spatial allocation, optimal
  weight direction).
  \textbf{Combined}: $u_\text{best}\cdot\mathbf{w}_1^T$ (Theorem~\ref{thm:separation}
  optimum with best-sign selection of $\pm\sqrt{\tilde{d}}/\|\sqrt{\tilde{d}}\|$).
  All variants share the same Frobenius budget $\|\Delta\mathbf{X}\|_F=\delta$.
  Random vectors are seeded per-seed (\texttt{seed+1000} for
  $w_\text{rand}$, \texttt{seed+2000} for $u_\text{rand}$).%
}
\label{tab:ablation}
\renewcommand{\arraystretch}{1.2}
\footnotesize
\begin{tabular}{@{}l c c c c c c@{}}
\toprule
\textbf{Dataset} & \textbf{Features}
  & \textbf{Clean}
  & \textbf{Random}
  & \textbf{Graph-only}
  & \textbf{Weight-only}
  & \textbf{Combined}
  \\
\midrule
  Cora     & binary     & 82.6 & 82.6 & 82.6 & 81.1 & \best{65.4} \\
  CiteSeer & binary     & 74.6 & 74.6 & 74.6 & 72.2 & \best{62.1} \\
  Cora-ML  & continuous & 85.5 & 85.4 & 85.4 & 74.3 & \best{28.1} \\
  PubMed   & continuous & 86.1 & 86.1 & 86.0 & 75.0 & \best{42.4} \\
\bottomrule
\end{tabular}
\end{table}

On binary-feature graphs (Cora, CiteSeer), randomizing \emph{either} factor collapses the attack to the random-baseline floor: Graph-only and Weight-only both land within 2.5\,pp of clean accuracy, while Combined drives accuracy down by $17.2$\,pp (Cora) and $12.5$\,pp (CiteSeer). This is the multiplicative signature predicted by the rank-one decomposition $\Delta\mathbf{X}^* = \sqrt\gamma\cdot\mathbf{v}_1\mathbf{w}_1^T$: the attack's strength is the \emph{product} of two spectral contributions, and disabling either factor sacrifices almost all of the damage.

On continuous-feature graphs (Cora-ML, PubMed), the two factors are not symmetric. Graph-only stays at the random floor ($\leq\!0.2$\,pp from clean), but Weight-only already produces a meaningful attack ($11.2$\,pp drop on Cora-ML and $11.1$\,pp on PubMed), showing that the weight direction $\mathbf{w}_1$ carries substantial adversarial signal even when paired with a spatially incoherent allocation. Adding the optimal degree-weighted allocation amplifies this further, to $57.4$\,pp (Cora-ML) and $43.7$\,pp (PubMed): roughly $5\times$ and $4\times$ multiplicative gains from aligning the spatial factor with $\sqrt{\tilde{\mathbf{d}}}$. This asymmetry between graph and weight factors refines the interpretation of the transfer results in Section~\ref{sec:transfer}: near-identical transfer on Cora-ML and PubMed is driven primarily by the portability of $\mathbf{w}_1$ across architectures (which accounts for the first ${\sim}11$\,pp of damage by itself), with the architecture-invariant spatial factor providing the additional multiplicative amplification when an architecture shares $\mathbf{w}_1$ geometry.

\section{Binary-Feature Cross-Architecture Transfer}
\label{app:binary_transfer}

Table~\ref{tab:binary_transfer} reports cross-architecture transfer on the binary-feature citation networks (Cora, CiteSeer), the negative complement to the continuous-feature transfer of Figure~\ref{fig:transfer}. All three attacks lose substantial efficiency relative to the source GCN. SepAttack degrades the most, consistent with the separation-principle prediction: on binary features the weight factor $\mathbf{w}_1$ does not share across architectures, so the rank-one attack is tied to the source GCN's weights in a way that gradient-based attacks (PGD-40, FGM)---which re-optimize at the single-step level---partially escape.

\begin{table}[t]
\centering
\caption{Cross-architecture transfer on binary-feature citation networks at $\delta=10$, averaged over 5~seeds. Cells report transfer ratio $\pm$~std (accuracy drop on target / drop on source GCN, \%). Last column: accuracy drop on the source GCN itself (pp), the denominator used for the ratios.}
\label{tab:binary_transfer}
\renewcommand{\arraystretch}{1.15}
\footnotesize
\begin{tabular}{@{}ll cccc c@{}}
\toprule
 & & \multicolumn{4}{c}{\textbf{Transfer ratio (\%) to target architecture}} & \textbf{Source} \\
\cmidrule(lr){3-6}
\textbf{Dataset} & \textbf{Attack} & \textbf{GAT} & \textbf{SAGE} & \textbf{APPNP} & \textbf{JKNet} & \textbf{drop (pp)} \\
\midrule
\multirow{3}{*}{Cora} & SepAttack (our attack) & 50\% $\pm$ 9 & 27\% $\pm$ 13 & 52\% $\pm$ 15 & 73\% $\pm$ 38 & 17.2 $\pm$ 2.8 \\
 & PGD-40 & 79\% $\pm$ 5 & 75\% $\pm$ 9 & 76\% $\pm$ 7 & 92\% $\pm$ 4 & 23.0 $\pm$ 2.0 \\
 & FGM    & 77\% $\pm$ 6 & 70\% $\pm$ 9 & 70\% $\pm$ 6 & 91\% $\pm$ 2 & 16.0 $\pm$ 1.4 \\
\midrule
\multirow{3}{*}{CiteSeer} & SepAttack & 22\% $\pm$ 12 & 31\% $\pm$ 12 & 20\% $\pm$ 10 & 43\% $\pm$ 20 & 12.5 $\pm$ 2.5 \\
 & PGD-40 & 66\% $\pm$ 6 & 77\% $\pm$ 6 & 64\% $\pm$ 3 & 88\% $\pm$ 10 & 24.1 $\pm$ 1.7 \\
 & FGM    & 63\% $\pm$ 6 & 75\% $\pm$ 7 & 57\% $\pm$ 1 & 86\% $\pm$ 10 & 18.1 $\pm$ 1.9 \\
\bottomrule
\end{tabular}
\end{table}

\end{document}